\documentclass[10pt,twocolumn]{IEEEtran}
\IEEEoverridecommandlockouts

\usepackage{cite}
\usepackage{amsmath}
\usepackage{amsfonts}
\usepackage{amssymb}
\usepackage{algorithm}
\usepackage{algpseudocode}
\usepackage{textcomp}
\usepackage{graphicx}
\usepackage{url}
\usepackage{comment}
\usepackage{booktabs}
\usepackage{subfig}
\usepackage{xcolor}

\usepackage{mathtools}

\DeclarePairedDelimiter{\diagfences}{(}{)}

\newcommand{\vectorform}{\operatorname{vectorform}\diagfences}
\newcommand{\matrixform}{\operatorname{matrixform}\diagfences}

\newcommand{\diag}{\operatorname{diag}\diagfences}
\newcommand{\tr}{\operatorname{tr}\diagfences}

\DeclareMathOperator*{\argmin}{argmin}

\newcommand{\ts}{\textsuperscript}

\newtheorem{theorem}{Theorem}[section]
\newtheorem{lemma}[theorem]{Lemma}

\newenvironment{proof}[1][Proof]{\begin{trivlist}
		\item[\hskip \labelsep {\bfseries #1}]}{\end{trivlist}}

\newcommand{\qed}{\nobreak \ifvmode \relax \else
	\ifdim\lastskip<1.5em \hskip-\lastskip
	\hskip1.5em plus0em minus0.5em \fi \nobreak
	\vrule height0.75em width0.5em depth0.25em\fi}

\def\BibTeX{{\rm B\kern-.05em{\sc i\kern-.025em b}\kern-.08em
    T\kern-.1667em\lower.7ex\hbox{E}\kern-.125emX}}

\begin{document}

\title{Learning an Interpretable Graph Structure in Multi-Task Learning}

\author{Shujian~Yu, Francesco Alesiani, Ammar Shaker and Wenzhe Yin%
\thanks{Shujian Yu, Francesco Alesiani and Ammar Shaker are with the NEC Laboratories Europe, $69115$ Heidelberg,
  Germany. (email: \{Shujian.Yu,Francesco.Alesiani,Ammar.Shaker\}@neclab.eu)}%
\thanks{Wenzhe Yin is with the Heidelberg University, $69117$ Heidelberg, Germany.
(email: Wenzhe.Yin@stud.uni-heidelberg.de)}}%

\maketitle

\begin{abstract}
We present a novel methodology to jointly perform multi-task learning and infer intrinsic relationship among tasks by an interpretable and sparse graph. Unlike existing multi-task learning methodologies, the graph structure is not assumed to be known a priori or estimated separately in a preprocessing step. Instead, our graph is learned simultaneously with model parameters of each task, thus it reflects the critical relationship among tasks in the specific prediction problem. We characterize graph structure with its weighted adjacency matrix and show that the overall objective can be optimized alternatively until convergence. We also show that our methodology can be simply extended to a nonlinear form by being embedded into a multi-head radial basis function network (RBFN). Extensive experiments, against six state-of-the-art methodologies, on both synthetic data and real-world applications suggest that our methodology is able to reduce generalization error, and, at the same time, reveal a sparse graph over tasks that is much easier to interpret.
\end{abstract}


\graphicspath{{figures/}}

\section{Introduction}
Multi-task learning is a subfield of machine learning in which individual models for performing potentially related tasks are learned jointly~\cite{zhang2017overview}. The advantages of multi-task learning are especially pronounced in situations where there is strong correlation between information-rich tasks and information-poor tasks~\cite{zhang2010convex}. By borrowing strength across tasks, it may be possible to reduce the overall generalization error. With this characteristic, multi-task learning has been used successfully across all applications of machine learning, from speech, natural language processing to computer vision~\cite{ruder2017overview}.

Since multi-task learning aims to improve the performance of a task with the help of other related tasks, a central problem is to accurately characterize relationship among multiple tasks. When structure about multiple tasks is available, e.g., task-specific descriptors~\cite{bonilla2007kernel} or a task similarity graph~\cite{evgeniou2004regularized,evgeniou2005learning}, one can impose regularizations into the learning formulation to penalize hypotheses that are not consistent with the given structure. However, in real-world scenarios, the structure information is always unavailable or hard to be obtained.

Modeling task relationship with a task covariance matrix $C$  or a task precision matrix $P$ ($a.k.a.$, the inverse of $C$) is a common strategy for existing multi-task structure learning methodologies (e.g.,~\cite{zhang2010convex,gonccalves2016multi,goncalves2019bayesian,zhao2017efficient,ciliberto2015convex}). Although either $P$ or $C$ carries partial correlation between pairwise tasks, there is no guarantee that those matrices can be transformed into a valid graph Laplacian~\cite{dong2016learning}.

Apart from a few early attempts (e.g.,~\cite{jacob2009clustered,zhou2011clustered}) that infer task relationship with disjoint clusters or subgroups, there has been very limited work on the joint learning of multiple tasks and a concrete data structure (e.g., graph, tree) across tasks~\cite{han2015learning}. One possible reason is that the problem of learning of a valid graph from observation data alone is still a challenging problem in both signal processing and machine learning communities~\cite{dong2016learning,chepuri2017learning,shen2017kernel}. Despite such difficulty, a graph structure across tasks improves the model interpretability and also enables many other downstream applications, such as the identification of outlier tasks and the visualization of task topology~\cite{zhang2017overview}.

In this paper, we propose a novel methodology for simultaneously learning of model parameters in each task and a sparse graph structure over tasks. Specifically, instead of learning a task covariance or precision matrix,
we resort to learn a weighted adjacency matrix $A$ to characterize a valid graph. We show how to integrate the learning of a weighted adjacency matrix and the learning of model parameters in each task to form a joint objective. We also show how this joint objective can be optimized alternatively. We then show our methodology can be seamlessly embedded into a multi-head radial basis function network (RBFN) to form a nonlinear model. We finally perform experiments to demonstrate the superiority of our methodology over other state-of-the-art (SOTA) ones.

\emph{Notation}: We use lowercase letters (e.g., $y$) for scalars, lowercase bold letters (e.g., $\mathbf{x}$) for vectors, and uppercase letters for matrix (e.g., $W$). $S^T$ denotes the space of symmetric $T\times T$-matrices, $S_+^T$ stands for the cone of symmetric $T\times T$-positive semidefinite matrices, and $\mathbf{M}_{d,T}$ is the space of real $d \times T$-matrices.

\section{Background Knowledge}
\subsection{Problem Formulation}

Suppose we are given $T$ learning tasks, where in each task we have access to a training set $\mathcal{D}_t$ with $N_t$ data instances $\{(\mathbf{x}_t^i,y_t^i):i=1,\cdots,N_t, t=1,\cdots,T\}$. In this work, we focus on the regression setting where $\mathbf{x}_t^i\in\mathcal{X}_t\subseteq\mathbb{R}^d$ and $y_t^i\in\mathbb{R}$. These tasks may be viewed as drawn from an unknown joint distribution of tasks, which is the source of the bias that relates the tasks. Multi-task learning aims to learn from each training set $\mathcal{D}_t$ a prediction model $f_t(\mathbf{w}_t,\cdot):\mathcal{X}_t\rightarrow\mathbb{R}$ with parameter $\mathbf{w}_t$ such that the task relationship is taken into consideration and the overall generalization error is small.


In what follows, we first assume a linear model in each task, i.e., $f_t(\mathbf{w}_t,\mathbf{x})=\mathbf{w}_t^T\mathbf{x}$. We will then discuss its nonlinear extension with the form $f_t(\mathbf{w}_t,\mathbf{x})=\mathbf{w}_t^Tg(\mathbf{x};\mathbf{\theta})$, where $g(\mathbf{x};\mathbf{\theta}):\mathbb{R}^d\rightarrow\mathbb{R}^p$ denotes a neural network with learnable parameter $\mathbf{\theta}$ that defines a nonlinear transformation of the input from $\mathbb{R}^d$ to $\mathbb{R}^p$.

\subsection{Related Work}

The joint learning of multiple tasks and their structure was initiated in  Multi-Task Gaussian Process (MTGP) prediction~\cite{bonilla2008multi} and Multi-Task Relationship Learning (MTRL)~\cite{zhang2010convex}, in which the task relationship is characterized by a task covariance matrix.
Unlike MTGP and MTRL, Multitask Sparse Structure Learning (MSSL)~\cite{gonccalves2016multi} directly learns a task precision matrix using a regularized Gaussian graphical model. On the other hand, the recently proposed Bayesian Multitask with Structure Learning (BMSL)~\cite{goncalves2019bayesian} imposes sparsity constraints (guided with prior information) on the inverse of covariance matrix to improve model interpretability.



Organizing multiple tasks with a concrete data structure (e.g., graphs, trees or disjoint clusters) is an alternative to infer their relationship.
One notable example is the TAsk Tree (TAT)~\cite{han2015learning}, in which the authors decomposed the parameter matrix into multiple layers and devised sequential constraints to make the distance between the parameters in the component matrices corresponding to each pair of tasks decrease over layers.
Despite the great potential of a tree structure and the solid theoretical guarantee behind optimization, TAT itself does not output a valid tree topology. Instead, one needs some post-hoc procedures (like the normalized graph cut~\cite{shi2000normalized} in each layer) to construct a tree-like architecture. The generated tree helps to group tasks based on model closeness, but it does not identify critical structures or connections among tasks.

In terms of a graph, one can view each task as a node, and two nodes are connected if the two tasks are related. Although the non-zero entries in the precision matrix carry partial correlations between two tasks~\cite{rue2005gaussian}, there is no guarantee that the learned precision matrix (from MSSL, BMSL, etc.) contains only non-positive off-diagonal entries and is zero row-sum, whereas both constraints are necessary to define a valid graph Laplacian~\cite{dong2016learning}. On the other hand, most existing graph structure multi-task learning methodologies assume that the graph topology is known a priori (e.g.,~\cite{evgeniou2005learning,nassif2018distributed,nassif2018regularization}) or can be simply predefined from the observation data (e.g.,~\cite{he2019efficient,chen2010graph}). Unfortunately, in many real-world scenarios, a graph structure is either unavailable or hard to be predefined correctly due to its complex nature~\cite{argyriou2013learning}.

Although the interpretable machine learning has gained increasing attention in recent years, existing interpretable multi-task models are always application-specific and feature-level based, i.e., revealing how much each feature contributes to the regression/classification result. For example, in industrial process control, the interpretability can be obtained by using an attention mechanism to determine which sensor influences the performance of product quality prediction~\cite{yeh2019interpretable}. By contrast, we target ``relationship interpretability" by enforcing the model to learn a sparse graph over tasks, which can give us insight about the relationships between tasks~\cite{zhang2017overview}.


Our work is similar in spirit to the graph fused Lasso (GFL)~\cite{chen2010graph} and the Convex Clustering Multi-Task Learning (CCMTL)~\cite{he2019efficient}. However, both GFL and CCMTL separate multi-task learning and graph structure estimation. Specifically, CCMTL predefines a static $k$-NN graph by measuring the $\ell_2$ distance on the parameters of prediction models learned independently for each task, whereas GFL generates the graph simply by evaluating the correlation coefficient between the response variables of pairwise tasks. More recently,~\cite{liu2019learning} proposes a neural network based graph multi-task learning framework for natural language processing (NLP) applications with input of text sequences, in which the task relatedness is not static but changes over time. The authors learn task communications by taking ideas from message passing~\cite{berendsen1995gromacs}, in which a directed (and usually dense) graph is defined over tasks. Our work is not designed for text sequences in a dynamical environment. Moreover, we aim to learn an undirected graph that is sparse and much easier to interpret.

\section{The Problem of Learning a Graph in MTL}
When learning linear models, each task is represented as a predictive function $\mathbf{w}_t^T\mathbf{x}_t\mapsto y_t$, where $\mathbf{w}_t$ is the regression parameter. The multi-task regression problem with a regularization $\Omega$ on the model parameters is defined as:
\begin{equation}\label{eq:obj1}
\min_{W} \sum_{t=1}^T\|\mathbf{w}_t^T\mathbf{x}_t-y_t\|_2^2 + \gamma\Omega(W): W\in \mathbf{M}_{d,T},
\end{equation}
where $W=[\mathbf{w}_1,\mathbf{w}_2,\cdots,\mathbf{w}_T]$ consists of columns $\mathbf{w}_t$.


Graph regularization is a natural choice in Eq.~(\ref{eq:obj1}), which is defined as:
\begin{equation}\label{eq:obj2}
\Omega(W) = \sum_{i=1}^T\sum_{j\in\mathcal{N}_i}A_{ij}\|\mathbf{w}_i-\mathbf{w}_j\|_2^2,
\end{equation}
where $A_{i,j}$ encodes the relatedness between task $i$ and task $j$, $\mathcal{N}_i$ is the set of neighbors of $i$, i.e., the set of nodes connected to task $i$ by an edge.
Let us define the pairwise distance matrix $Z$ as $Z_{ij} = \|\mathbf{w}_i-\mathbf{w}_j\|_2^2$, the quadratic penalty term in Eq.~(\ref{eq:obj2}) is equivalent to\footnote{$L = D-A$ is the graph Laplacian matrix, where $D = \diag{\mathbf{d}}$ is the diagonal matrix formed by the degrees of the vertices $d_i = \sum_{j=1}^T A_{ij}$.}~\cite{zhou2004regularization}:
\begin{equation} \label{eq:reg}
\sum_{i=1}^T\sum_{j\in\mathcal{N}_i}A_{ij}\|\mathbf{w}_i-\mathbf{w}_j\|_2^2 = \|A\circ Z\|_{1,1} = 2\tr{W^TLW}.
\end{equation}


Therefore, Eq.~(\ref{eq:obj2}) can be simply expressed as:
\begin{equation}\label{eq:obj4}
\min_{W} g(W) + \gamma \|A\circ Z\|_{1,1}: W\in \mathbf{M}_{d,T}, Z_{ij} = \|\mathbf{w}_i-\mathbf{w}_j\|_2^2,
\end{equation}
with $g(W)=\sum_{t=1}^T\|\mathbf{w}_t^T\mathbf{x}_t-y_t\|_2^2$.

Thus, the problem we are going to address is how to learn simultaneously the model parameters $W$ of $T$ tasks and the graph of tasks via its weighted adjacency matrix $A$ in a joint manner with the following objective:
\begin{equation}\label{eq:obj_simple}
\begin{split}
\min_{W,A} &~g(W) + \gamma \|A\circ Z\|_{1,1} + f(A): \\
& W\in \mathbf{M}_{d,T}, A\in \mathcal{A}, Z_{ij} = \|\mathbf{w}_i-\mathbf{w}_j\|_2^2,.
\end{split}
\end{equation}

\subsection{Learning Proper Graph Structure: the Role of $f(A)$ }
$f(A)$ has to play two important roles: (1) prevent $A$ from going to the trivial solution $A=0$ and (2) impose further structure using prior information on $A$.


The space of all valid weighted adjacency matrix $A$ is given by:
\begin{equation}
\mathcal{A}=\left\{A\in S^T:\ (\forall i\neq j)\ A_{ij}\geq0,\ \diag{A}=0\right\},
\end{equation}
which can be viewed as a relaxation of the search space defined by either task covariance matrix or graph Laplacian, both of which are in $S_{+}^T$.


To promote the discovery of connected graph, we encourage each node to be connected to at least another node. Further, we want to control the sparseness of the resulting graph. Motivated by recent progress in graph signal processing (e.g.,~\cite{dong2016learning,chepuri2017learning,kalofolias2016learn}), we use the following model with parameters $\alpha>0$ and $\beta>0$ to control the shape of the graph:
\begin{equation}\label{eq:obj_adj}
\min_{A\in \mathcal{A}} = \|A\circ Z\|_{1,1} - \alpha \mathbf{1}^T\log(A\mathbf{1}) + \beta\|A\|_F^2: Z_{ij} = \|\mathbf{w}_i-\mathbf{w}_j\|_2^2,
\end{equation}
where $\mathbf{1}=[1,\cdots,1]^T$.


The logarithmic barrier acts on the node degree vector $A\mathbf{1}$. This means that it forces the degrees to be positive, but does not prevent edges from becoming zero. This improves the overall connectivity of the graph, without compromising sparsity. Note however, that adding solely a logarithmic term ($\beta=0$) leads to very sparse graphs, and changing $\alpha$ only changes the scale of the solution and not the sparsity pattern. For this reason, we add the term $\beta\|A\|_F^2$.

Combine Eqs.~(\ref{eq:obj2}), (\ref{eq:obj_simple}), and (\ref{eq:obj_adj}), our final objective is given by:
\begin{equation}\label{eq:obj_complex}
\begin{split}
\min_{W \in \mathbf{M}_{d,T} ,A \in \mathcal{A}} &~ \sum_{t=1}^T\|\mathbf{w}_t^T\mathbf{x}_t-y_t\|_2^2 + \gamma \|A\circ Z\|_{1,1} \\
&- \alpha \mathbf{1}^T\log(A\mathbf{1}) + \beta\|A\|_F^2.
\end{split}
\end{equation}

\subsection{Graph Adjacency Multi-Task Learning (GAMTL)}


The objective (\ref{eq:obj_complex}) is bi-convex in $W$ and $A$ (see \textit{Theorem}~\ref{th:biconvex}). We thus exploit this property and define the GAMTL in Algorithm~\ref{alg:MTL}, which alternates between minimization of $W$ and minimization of $A$.

\begin{algorithm}[!htbp]
\caption{Graph Adjacency Multi-Task Learning}
\small
\label{alg:MTL}
\begin{algorithmic}[1]
\Require
$W^0\in \mathbf{M}_{d,T}$;
$A^0$;
$\{(\mathbf{x}_t^i,y_t^i):i=1,\cdots,N_t, t=1,\cdots,T\}$.
\Ensure
$W^*$;
$A^*$.
\For {$k=1,2,\cdots$}
\State $W\leftarrow \argmin \big\{\sum_{t=1}^T\|\mathbf{w}_t^T\mathbf{x}_t-y_t\|_2^2 + \gamma\|A\circ Z\|_{1,1}: W\in \mathbf{M}_{d,T}, Z_{ij} = \|\mathbf{w}_i-\mathbf{w}_j\|_2^2\big\}$
\State $A\leftarrow \argmin \|A\circ Z\|_{1,1} - \alpha \mathbf{1}^T\log(A\mathbf{1}) + \beta\|A\|_F^2$
\EndFor
\end{algorithmic}
\end{algorithm}


For a faster convergence, the initial weight matrix $W^0$ consists of prediction models learned independently from each task, and the initial weighted adjacency matrix $A^0$ is a fully connected graph in which the edge weight is defined as the $\ell_2$ norm over initial model parameters. In our implementation, $W$ is updated with the Combinatorial Multigrid (CMG) solver~\cite{koutis2011combinatorial}, $A$ is solved by the primal dual algorithm~\cite{komodakis2015playing} as adopted in~\cite{kalofolias2016learn}.

\subsubsection{Solving for $W$} \label{sec:solve_w}

The first problem of Algorithm~\ref{alg:MTL} is obtained as the solution of
$\sum_{t=1}^T\|\mathbf{w}_t^T\mathbf{x}_t-y_t\|_2^2 + \sum_{i=1}^{T}\sum_{j\in\mathcal{N}_i}A_{ij}\|\mathbf{w}_i-\mathbf{w}_j\|_2^2$, which, for \textit{Lemma}~\ref{th:quadraticW}, is quadratic in $W$ and can be solved efficiently due to the sparseness of the variables.
\subsubsection{Solving for $A$}
The computation of $A$ is described in Algorithm~\ref{alg:primaldual} \cite{kalofolias2016learn}\footnote{$[\mathbf{x}]_+$ is the positive component of $\mathbf{x}$ and operations are performed element-wise. $^+$ denotes the update value. $\mathbf{w}$ is the upper part of $A$, thus enforcing $A$ to be symmetric.}, where the operator $S$ is defined such that $A\mathbf{1}=S\mathbf{w}$ and $\mathbf{w}$ is the vector form of $A$.

\begin{algorithm}[!htbp]
\caption{Primal-Dual algorithm for $A$}
\small
\label{alg:primaldual}
\begin{algorithmic}[1]
\Require
$A^0,\gamma,\alpha,\beta,S,Z,\epsilon$
\State $\mathbf{w}\gets \vectorform{A^0}$, $\mathbf{z}\gets \vectorform{Z}$
\State $\mathbf{v}=S\mathbf{w}$
\While {$||\mathbf{q}-\mathbf{y}||/||\mathbf{w}|| > \epsilon \lor ||\bar{\mathbf{q}}-\bar{\mathbf{y}}||/||\mathbf{v}|| > \epsilon  $}
\State $\mathbf{y}^+=\mathbf{w}-\gamma(2\beta \mathbf{w} + S^T\mathbf{v})$, $\bar{\mathbf{y}}^+=\mathbf{v}+\gamma S\mathbf{w}$
\State $\mathbf{p}^+=[\mathbf{y}-2 \gamma \mathbf{z}]_+$, $\bar{\mathbf{p}}^+=(\bar{\mathbf{y}}-\sqrt{\bar{\mathbf{y}}^2+4 \alpha \gamma})/2$
\State $\mathbf{q}^+=\mathbf{p}-\gamma (2 \beta \mathbf{p} +S^T\mathbf{p})$, $\bar{\mathbf{q}}^+=\bar{\mathbf{p}}+\gamma S\bar{\mathbf{p}}$
\State $\mathbf{w}^+=\mathbf{w}-\mathbf{y}+\mathbf{q}$, $\mathbf{v}^+=\mathbf{v}-\bar{\mathbf{y}}+\bar{\mathbf{q}}$
\EndWhile \\
\Return $A \gets \matrixform{\mathbf{w}}$
\end{algorithmic}
\end{algorithm}

\begin{theorem}\label{th:solutionW}
The problem $W = \arg \min_W \sum_{t=1}^T\|\mathbf{w}_t^T\mathbf{x}_t-y_t\|_2^2 + \sum_{i=1}^{T}\sum_{j\in\mathcal{N}_i}A_{ij}\|\mathbf{w}_i-\mathbf{w}_j\|_2^2$ reduces to solving a linear system.
\end{theorem}
\begin{proof}
Suppose $N=\sum_t N_t$, let us define $X = \diag{X_1,\dots,X_T} \in \mathbb{R}^{dT \times N}$ as a block diagonal matrix, define $W = [w^T_1,\dots w^T_T]^T \in \mathbb{R}^{d T \times 1}$ as a column vector, and define $Y = [y_1^T,\dots,y_T^T] \in \mathbb{R}^{1 \times N}$ as a row vector, the original problem can be rewritten as:
\begin{equation}\label{eq:appendix_obj2}
\min_V \|V^TX-Y\|_2^2 + \sum_{i=1}^{T}\sum_{j\in\mathcal{N}_i}A_{ij}\|V((\mathbf{e}_i-\mathbf{e}_j)\otimes I_d)\|_2^2,
\end{equation}
where $\mathbf{e}_i\in \mathbb{R}^T$ is an indicator vector with the $i$-th element set to $1$ and others $0$, and $I_d$ is an identity matrix of size $d\times d$.

Setting the derivative of Eq.~(\ref{eq:appendix_obj2}) equal to zero with respect to $V$, we obtain the following linear system:
\begin{equation}
    (B+C)V = D,
\end{equation}
where $B=(\sum_{i=1}^{T}\sum_{j\in\mathcal{N}_i}A_{ij}(\mathbf{e}_i-\mathbf{e}_j)(\mathbf{e}_i-\mathbf{e}_j)^T)\otimes I_d$, $C=XX^T$, and $D=XY^T$.
\end{proof}
\begin{lemma}\label{th:quadraticW}
The problem $W = \arg \min_W \sum_{t=1}^T\|\mathbf{w}_t^T\mathbf{x}_t-y_t\|_2^2 + \sum_{i=1}^{T}\sum_{j\in\mathcal{N}_i}A_{ij}\|\mathbf{w}_i-\mathbf{w}_j\|_2^2$ is quadratic in $W$.
\end{lemma}
\begin{proof}
\textit{Lemma}~\ref{th:quadraticW} follows from application of \textit{Theorem}~\ref{th:solutionW}.
\end{proof}
\begin{theorem} \label{th:biconvex}
The function $f(W,A)$ defined by Eq.~(\ref{eq:obj_complex}) is bi-convex  and analytic for $A\mathbf{1} > 0$.
\end{theorem}
\begin{proof}
$f(W,A) =  \sum_{t=1}^T\|\mathbf{w}_t^T\mathbf{x}_t-y_t\|_2^2 + \gamma \|A\circ Z\|_{1,1}
- \alpha \mathbf{1}^T\log(A\mathbf{1}) + \beta\|A\|_F^2$. The first term is quadratic in $W$.
The second term is also quadratic in $W$, but linear in $A$ (see Eq.~(\ref{eq:reg})). The third term is convex in $A$ for $A\mathbf{1} > 0$, while the last term is quadratic in $A$. We notice that the composition of the terms in $A$ form a convex function since is a composition of not decreasing and convex functions for $A \mathbf{1} >0$. It is also possible to show that $\nabla^2_{A}f(W,A) \succeq 0$ and block diagonal. Since  $f$ is quadratic in $W$ (see sec.\ref{sec:solve_w}) and convex in $A$, it is a bi-convex function for $A \mathbf{1} > 0$. Further $f$ is analytic since all terms are analytic functions for $A \mathbf{1} >0$.
\end{proof}
\begin{theorem} \label{th:convergence}
The sequence of $W^k,A^k$ generated by Algorithm~\ref{alg:MTL} converges, if bounded, to a first order stationary point\footnote{First order stationary point is defined for a function $f(x)$ as $x$ such that $\nabla_x f(x)=0$ , while second order stationary point it is a stationary point and $\nabla^2f(x) \succeq 0$.}, while the proximal version converges almost surely to the second-order stationary point.
\end{theorem}
\begin{proof}
The results follows from \textit{Theorem}~\ref{th:biconvex} and the results of \cite{xu2013block, li2019alternating} (\textit{Theorem}~$1$ and \textit{Theorem}~$2$).
\end{proof}
\subsection{Computational Complexity}
The computational complexity of Algorithm~\ref{alg:MTL} is defined by the complexity of computing $W$ and computing $A$.
Computing $W$ requires to solve the equation defined in Section~\ref{sec:solve_w}, whose complexity is $\mathcal{O}(d^3T^3+ d^2N + T^3)$\footnote{The first term is due to the inversion of $XX^T$, while the second the computation of $XX^T$, the third to compute $L = EE^T$, where $E$ consists of indicator vectors.}.
When the matrix $A$ is sparse, the solution can be efficiently computed using CMG~\cite{koutis2011combinatorial} and is shown to have a linear empirical complexity in $T$~\cite{he2019efficient}.
The complexity of computing $A$ is proportional to $\mathcal{O}(T^2)$ since it requires to evaluate function over a vector $\mathbf{w}$ (inside Algorithm~\ref{alg:primaldual}) of size $\mathcal{O}(T^2)$\footnote{The product $S\mathbf{w}$ is equivalent to sum the incident nodes for each node, whose complexity is $\mathcal{O}(1)$, when we consider only multiplications, even if $S$ is a matrix of size $T \times T^2$.}. The complexity could be reduce to $\mathcal{O}(T^2 + d^2N + d^3T)$ by enforcing sparsity on $A$ in solving for $W$, where the last term is only derived empirically \cite{he2019efficient}.
\subsection{Non-linear Extension}
We present nonlinear extension of GAMTL. Although kernel extension is straightforward~\cite{ciliberto2015convex}, this approach might lead to huge computational burden when the number of samples increases. Another more natural and expressive approach is to combine our joint objective and alternating optimization with parametrized nonlinear feature transformations, such as neural networks. More specifically, let $g(\mathbf{x};\theta): \mathbb{R}^d\mapsto\mathbb{R}^p$ be a neural network with learnable parameter $\theta$ that defines a nonlinear transformation of the input features from $\mathbb{R}^d$ to $\mathbb{R}^p$. We then add one more layer defined by parameter matrix $W$ on top of the nonlinear mapping.

Although a multi-layer perceptron (MLP) coupled with nonlinear activation functions always serves as the workhorse for nonlinear multi-task learning (e.g.,~\cite{zhao2017efficient,taylor2017personalized}), there is a large discrepancy between the stochastic gradient descent and our alternating optimization in Algorithm~\ref{alg:MTL}.
To this end, we resort to the standard RBFN. In contrast to a MLP, a RBFN can be trained layer-wisely: an unsupervised learning phase in the first layer (to determine RBF centers) followed by a linear supervised learning phase in the second layer. In this sense, one can simply integrate GAMTL into the second layer of a multi-head RBFN~\cite{liao2006radial}, in which the first layer is trained with $k$-means and the second layer is trained with Algorithm~\ref{alg:MTL}.

\begin{figure}[htbp]
\centering
\includegraphics[width=0.5\textwidth,page=3,trim=1cm 3cm 3cm 1cm,clip]{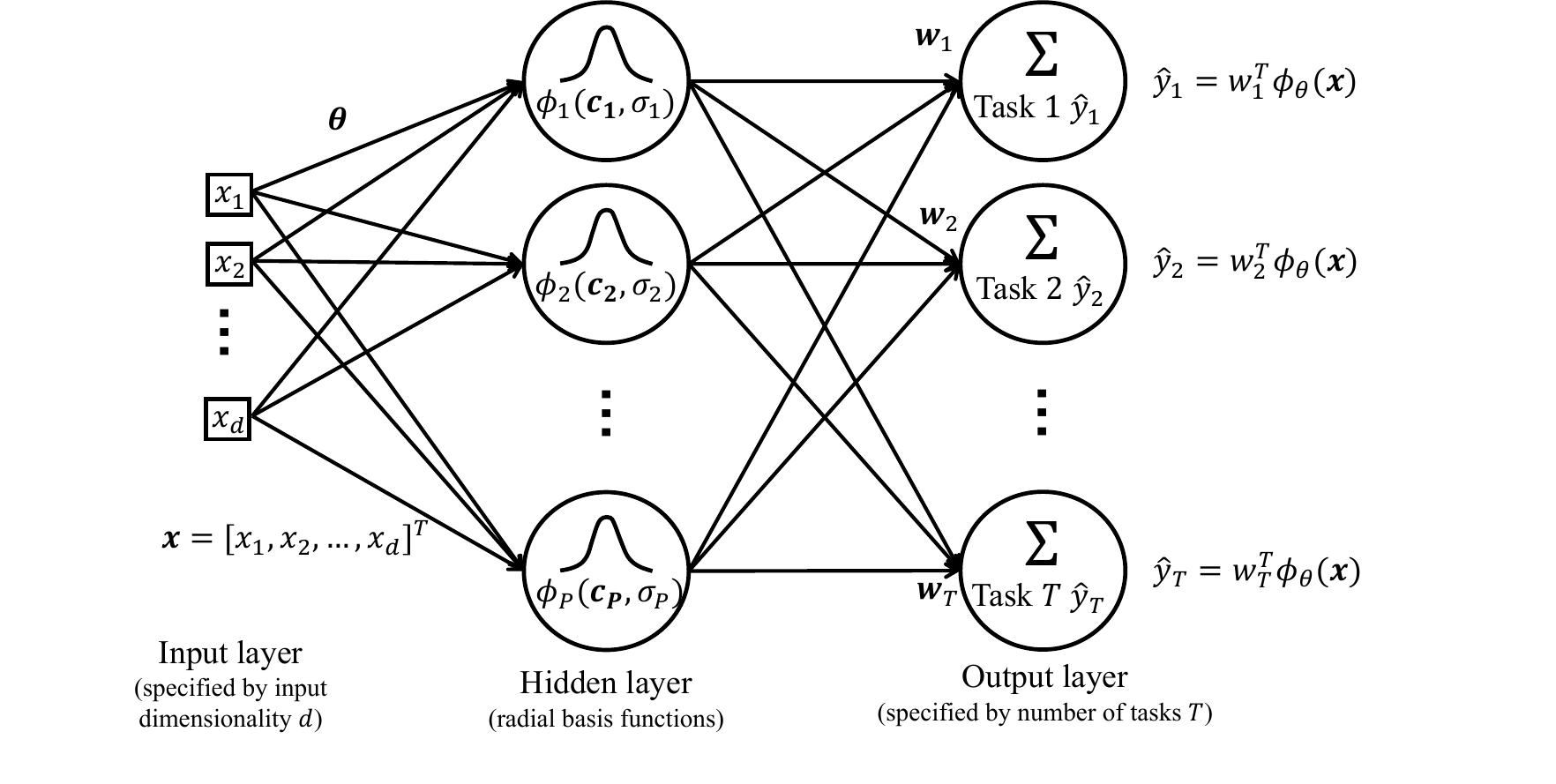}
\caption{A multi-head radial basis function (RBF) network. Each of the output nodes represents a unique task. Each task has its own hidden-to-output weights $\mathbf{w}_i$ but all the tasks share the same input-to-hidden weights $\theta$. $W=[\mathbf{w}_1,\mathbf{w}_2,\cdots,\mathbf{w}_T]$ is learned with Algorithm~\ref{alg:MTL}, whereas $\theta$ is learned with $k$-means.
{
The activation of hidden node $p$ is characterized by a RBF $\phi_p(\mathbf{x})=\phi_p(\|\mathbf{x}-\mathbf{c}_p\|,\sigma_p)$ with the centroid $\mathbf{c}_p$. The prediction $\hat{y}_t$ of $\mathbf{x}$ in the $t$-th task is $\hat{y}_t = w_t^T u_t $, where $u_t  = [\phi_p(\|\mathbf{x}_t^i-\mathbf{c}_p\|,\sigma_p)] =\phi_\theta(x_t) \in R^{p \times N_t}$ is the hidden features vector and $\theta = (\sigma_1,\dots,\sigma_P)$.
}
}
\label{fig:RBFN}
\end{figure}

We term this improvement RBF-GAMTL, which solves the following problem
\begin{equation}\label{eq:obj_rbf}
\begin{split}
\min_{W \in \mathbf{M}_{d,T} ,A \in \mathcal{A}} &~ \sum_{t=1}^T\|\mathbf{w}_t^T\phi_\theta(\mathbf{x}_t)-y_t\|_2^2 + \gamma \|A\circ Z\|_{1,1} \\
&- \alpha \mathbf{1}^T\log(A\mathbf{1}) + \beta\|A\|_F^2.
\end{split}
\end{equation}
Fig.~\ref{fig:RBFN} depicts the structure of our multi-head RBFN for multi-task learning. The RBF-GAMTL Alg.\ref{alg:RBF} extends Alg.\ref{alg:MTL} and solves Eq.\ref{eq:obj_rbf}, by first selecting the number of RBF centers (i.e., $P$), the centers ($\{c_p\}_{p=1}^P$) and the RBF kernels widths ($\theta = (\sigma_1,\dots,\sigma_P)$) using the optimal-width method~\cite{benoudjit2003kernel}.


\begin{algorithm}[!htbp]
\caption{RBF-GAMTL}
\small
\label{alg:RBF}
\begin{algorithmic}[1]
\Require
$W^0\in \mathbf{M}_{d,T}$; $A^0$; $\{(\mathbf{x}_t^i,y_t^i):i=1,\cdots,N_t, t=1,\cdots,T\}$.
\Ensure
$W^*$;
$A^*$.
\State $\{c_p\},\theta \gets \text{Optimal-Width}(\{x_t\})$ \Comment{Initialize $c_p$ using \cite{benoudjit2003kernel}}
\For {$k=1,2,\cdots$}
\State $W\leftarrow \argmin \big\{\sum_{t=1}^T\|w_t^T\phi_\theta(x_t)-y_t\|_2^2 + \gamma\|A\circ Z\|_{1,1}: W\in \mathbf{M}_{d,T}, Z_{ij} = \|w_i-w_j\|_2^2\big\}$
\State $A\leftarrow \argmin \|A\circ Z\|_{1,1} - \alpha \mathbf{1}^T\log(A\mathbf{1}) + \beta\|A\|_F^2$
\EndFor
\end{algorithmic}
\end{algorithm}


\section{Experiments}

We evaluate the performance of GAMTL and RBF-GAMTL against six SOTA multi-task learning methodologies (namely MTRL~\cite{zhang2010convex}, MSSL~\cite{gonccalves2016multi}, BMSL~\cite{goncalves2019bayesian},  TAT~\cite{han2015learning}, GFL~\cite{chen2010graph} and CCMTL~\cite{he2019efficient}) on both synthetic data and real-world applications. Among the six competitors, MTRL and BMSL learn a task covariance matrix, MSSL targets a task precision matrix which can be interpreted as a graph Laplacian. TAT infers a tree-like structure over layered components of weight matrix, in which the leaf nodes denote different tasks. On the other hand, CCMTL predefines a static $k$-NN graph, whereas GFL generates a graph by evaluating the correlation coefficient between response variables of pairwise tasks.
For a fair comparison, the hyper-parameters of all competing methods are selected with either author recommended values or via $5$-fold cross validation.


\subsection{Synthetic Data}

The synthetic data we consider here aims at demonstrating that our methodology is able to precisely infer the intrinsic structure of tasks and enjoys significant improvement on the interpretability of task relatedness against its competitors. We generate two synthetic data to illustrate our points. Each data contains $20$ linear regression tasks of input dimension $30$. For each task, the input variable $\mathbf{x}_t$ are generated $i.i.d.$ from an isotropic multivariate Gaussian distribution, i.e., $\mathbf{x}_t\sim\mathcal{N}(\mathbf{0},\mathbf{I}_{30})$. The corresponding output is generated as $y_t =\mathbf{w}_t^T\mathbf{x}_t+\epsilon$, where $\epsilon\sim\mathcal{N}(0,1)$. For simplicity, we assume all tasks share the same input.

In the first data (denote $\tt{Syn}~1$), the task parameters are chosen so that tasks $1$ to $12$ and tasks $13$ to $18$ form two groups, whereas tasks $19$ and $20$ are independent and significantly different from any other tasks (thus can be interpreted as outlier tasks). Specifically, parameters of tasks $1$ to $12$ are $\mathbf{w}_{1:12}=\mathbf{w}_{g_1}+0.1\times \mathbf{u}_{30}$, where $\mathbf{w}_{g_1}\sim\mathcal{N}(\mathbf{1},\mathbf{I}_{30})$ and $\mathbf{u}_{30}$ denotes a $30$-dimensional random vector with each element uniformly distributed between $[0,1]$. Similarly, parameters of tasks $13$ to $18$ are $\mathbf{w}_{13:18}=\mathbf{w}_{g_2}+0.1\times \mathbf{u}_{30}$, where $\mathbf{w}_{g_2}\sim\mathcal{N}(-\mathbf{1},\mathbf{I}_{30})$. Different from tasks $1$ to $18$, $\mathbf{w}_{19}\sim\mathcal{N}(\mathbf{0},\mathbf{I}_{30})$ and $\mathbf{w}_{20}\sim\mathcal{N}(\mathbf{0},10\ast\mathbf{I}_{30})$.

In the second data (denote $\tt{Syn}~2$), each task is only related to its neighbor tasks to manifest strong locality relationships. Specifically, $\mathbf{w}_1\sim\mathcal{N}(\mathbf{0},\mathbf{I}_{30})$, $\mathbf{w}_{2:30}$ shares the same regression coefficients with $\mathbf{w}_1$ on dimensions $3$ to $30$. However, the first two dimensions of $\mathbf{w}_{2:30}$ are generated by applying a rotation matrix of the form $R=\left[\begin{matrix}\cos(\theta)&-\sin(\theta)\\\sin(\theta)&\cos(\theta)\\\end{matrix}\right]$ to the first two dimensions of $\mathbf{w}_1$, in which $\theta$ is evenly spaced between $[0,2\pi]$. In this sense, $\mathbf{w}_{20}$ gets back to $\mathbf{w}_1$ and is thus also closely related to $\mathbf{w}_2$ and $\mathbf{w}_3$.

We train each method on a training set of $20$ samples in each task, and test their performances on a test set of $80$ samples in each task. Fig.~\ref{fig:graph_topology_data1} and Fig.~\ref{fig:graph_topology_data2} demonstrate the task relatedness learned by all competing methods on $\tt{Syn}~1$ and $\tt{Syn}~2$, respectively. As can be seen, our GAMTL is able to learn even the complex ``circular" task relations and that the obtained sparse weighted adjacency matrix $A$ enhances interpretability on task relations. By contrast, both MTRL and MSSL recover the groups of tasks or the dense ``circular" structure, but such relatedness is not as straightforward as a graph and is likely to be dominated by the main diagonal of the task covariance matrix. BMSL performs well in $\tt{Syn}~1$, but fails in $\tt{Syn}~2$.
A static $k$-NN graph in CCMTL is hard to identify outlier tasks (see Fig.~\ref{fig:graph_topology_data1}(f)), whereas a graph defined using the correlation coefficient is likely to overfit the underlying task relatedness (see Fig.~\ref{fig:graph_topology_data2}(e)).
On the other hand, TAT often identifies partial relations between tasks. For example, in $\tt{Syn}~2$, TAT correctly discovers that neighboring tasks are similar locally, but fails to unveil the global ``circular" structure.
The RMSE values over $10$ independent runs are reported in Table~\ref{Tab:synthetic}. In most of the cases, a precise task relationship also reduces the overall regression error. Interestingly, RBF-GAMTL does not show performance gain over its linear counterpart. This is probably because that a linear model is sufficiently powerful for linear data. Moreover, a neural network is liable to overfitting and results in poor generalization with small sample size.


\begin{table*}[!hbpt]
\centering
\caption{RMSE (mean$\pm$std) on $\tt{Syn}~1$ and $\tt{Syn}~2$ over $10$ independent runs. The best two performances are marked in bold and underlined, respectively.}\label{Tab:synthetic}
\begin{tabular}{ccccccccc}
\toprule
 & MTRL & MSSL & BMSL & TAT & GFL & CCMTL & \textbf{GAMTL} & \textbf{RBF-GAMTL} \\
\midrule
$\tt{Syn}~1$ & $5.996\pm0.859$  & $6.046 \pm 0.797$ & $8.361 \pm 0.992$ & $6.016\pm 0.919$ & $6.750 \pm 1.019$ & $\underline{5.609}\pm0.977$ & $\mathbf{5.595}\pm0.983$ & $7.100\pm0.465$\\
$\tt{Syn}~2$ & $3.451\pm0.563$  & $3.584 \pm 0.548$ & $4.533 \pm 0.915$ & $3.188\pm 0.652$ & $3.993 \pm 0.660$ & $\underline{3.175}\pm0.650$ & $\mathbf{3.164}\pm0.649$ & $4.092\pm 0.600$\\
\bottomrule
\end{tabular}
\end{table*}

\begin{figure*}
\centering
\subfloat[MTRL]{{\includegraphics[width=0.23\textwidth]{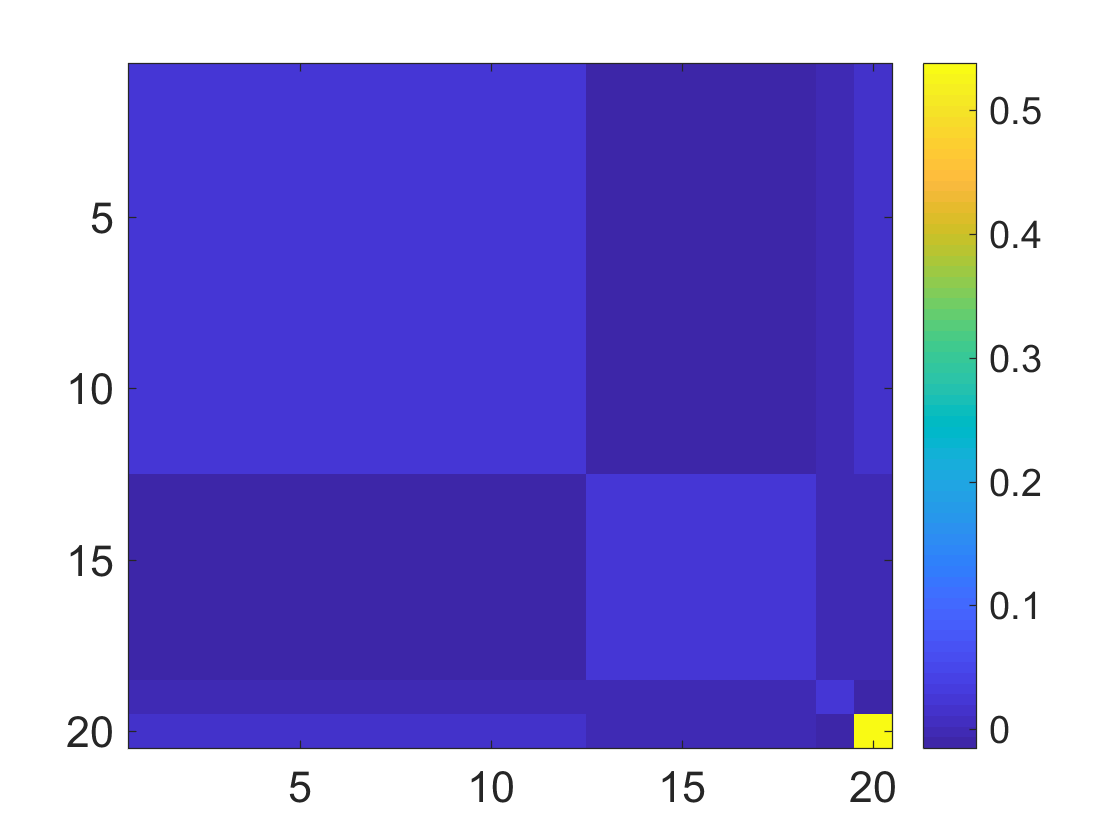} }}%
\hfill
\subfloat[MSSL]{{\includegraphics[width=0.23\textwidth]{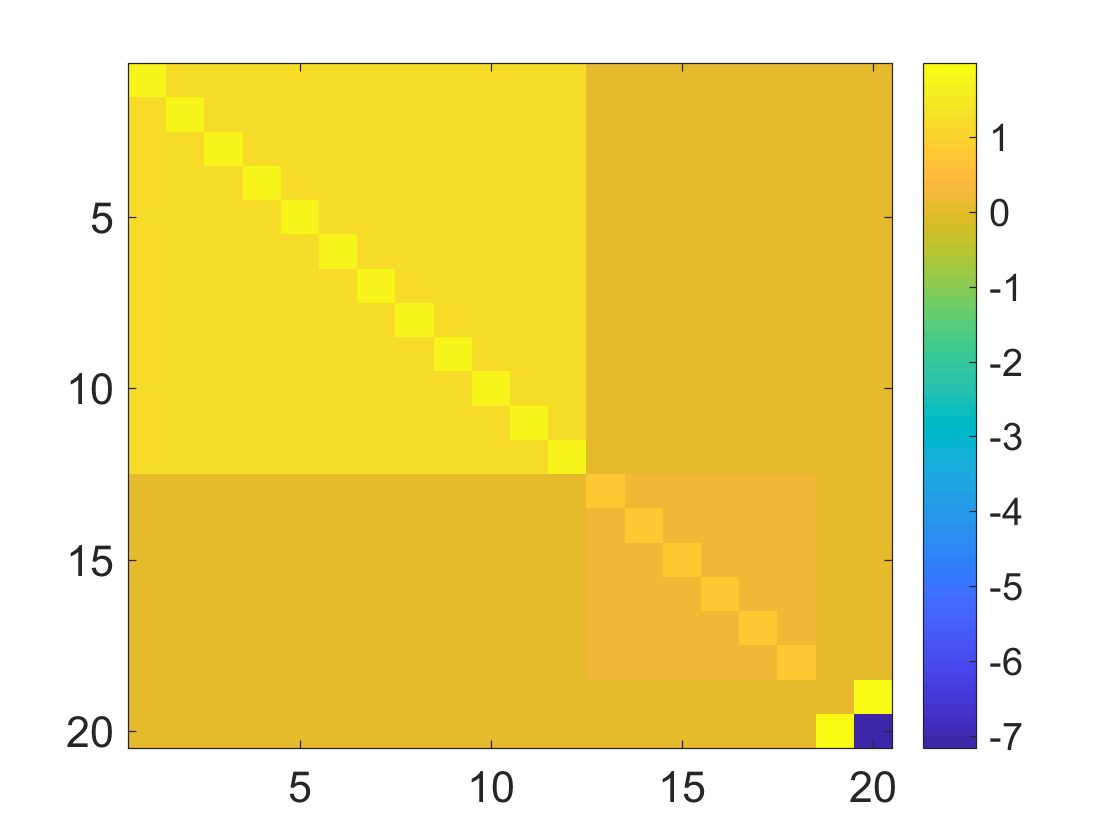} }}%
\hfill
\subfloat[BMSL]{{\includegraphics[width=0.23\textwidth]{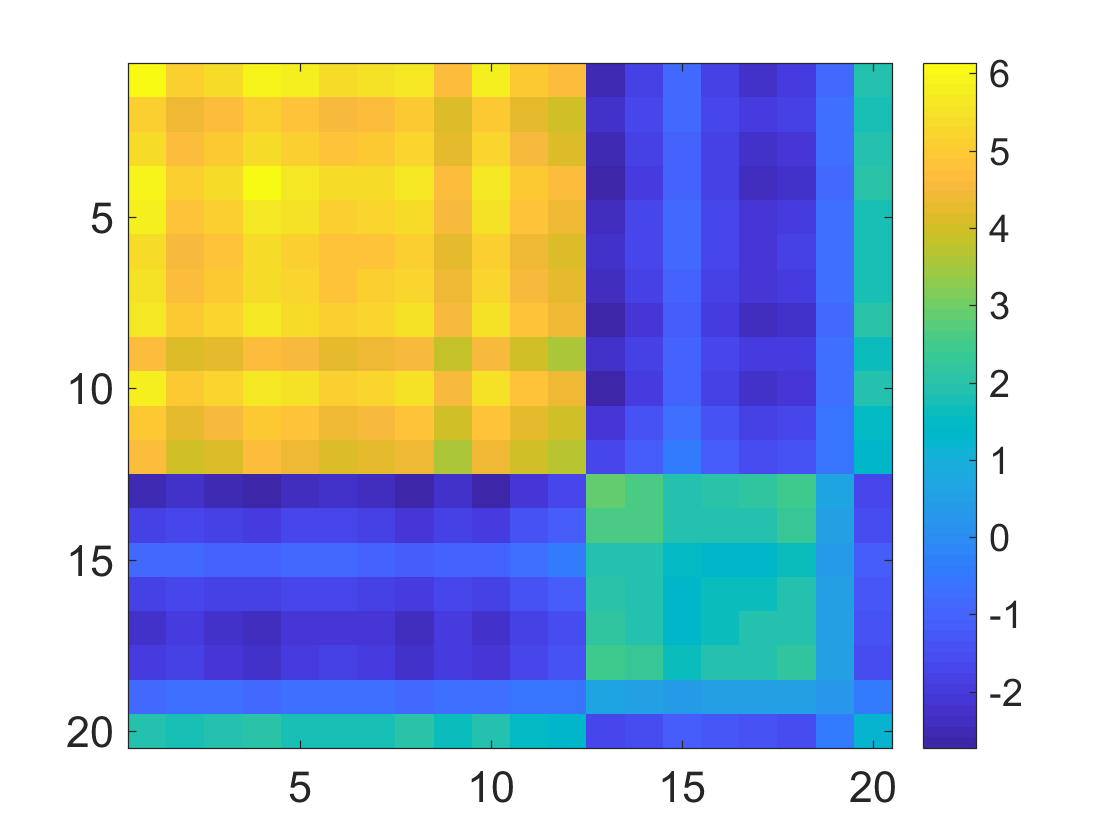} }}
\hfill
\subfloat[TAT]{{\includegraphics[width=0.18\textwidth]{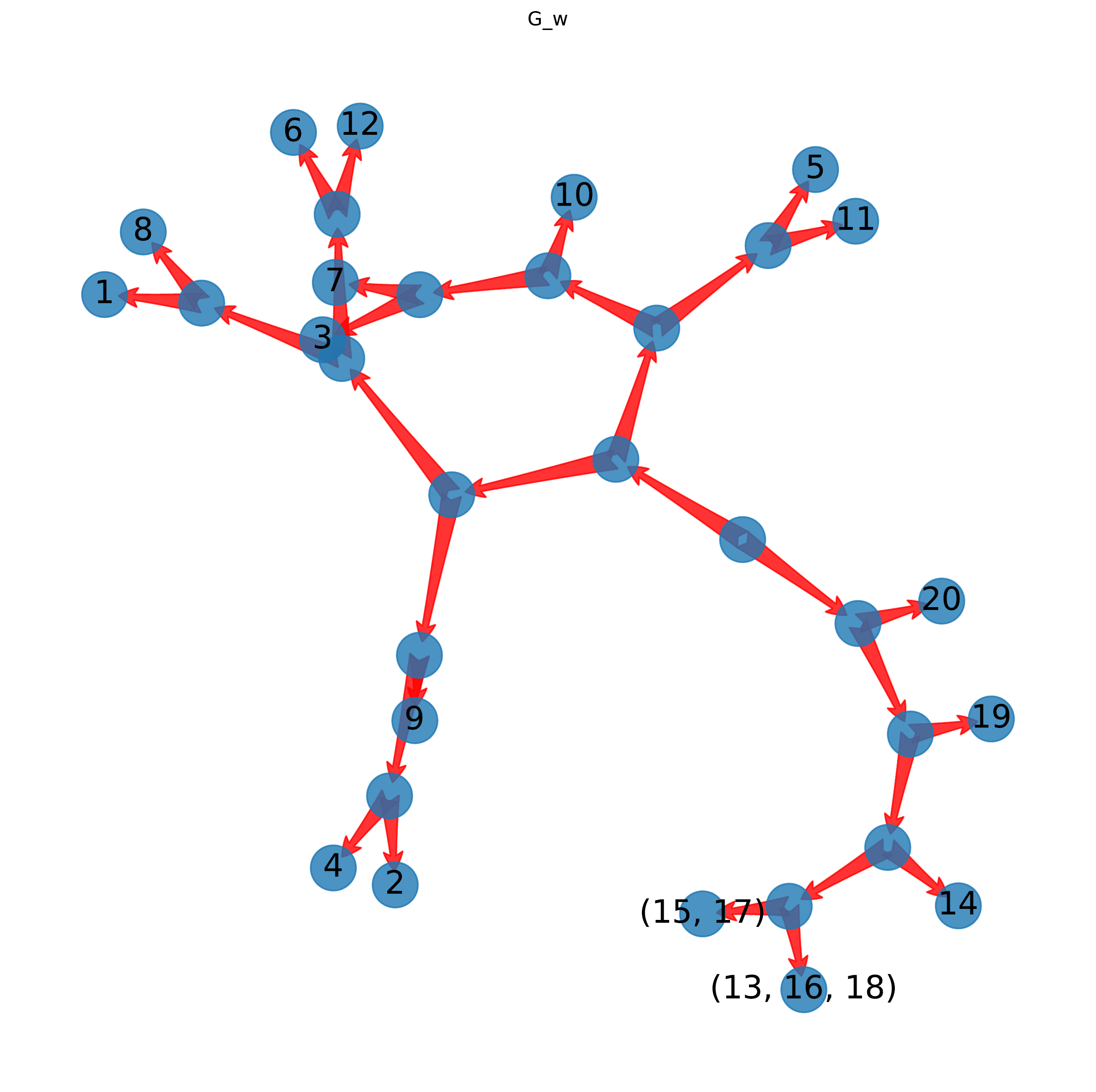} }}\\
\subfloat[GFL]{{\includegraphics[width=0.23\textwidth]{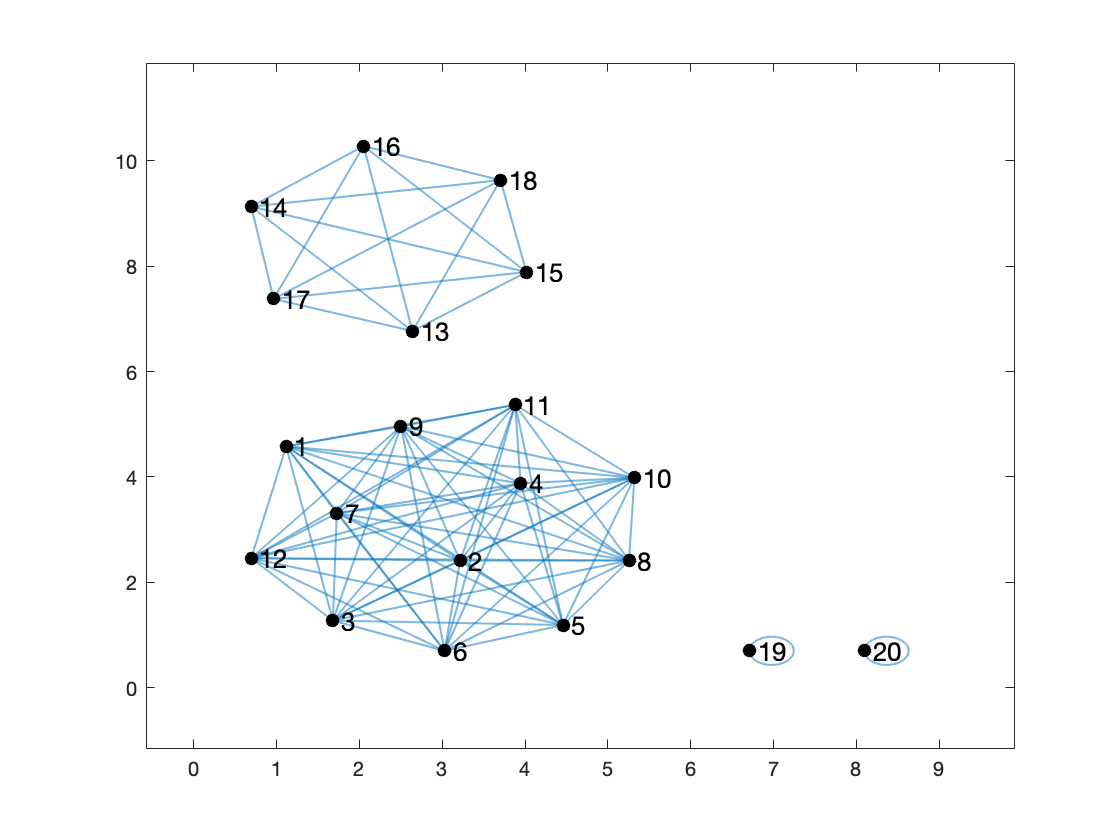} }}
\hfill
\subfloat[CCMTL]{{\includegraphics[width=0.23\textwidth]{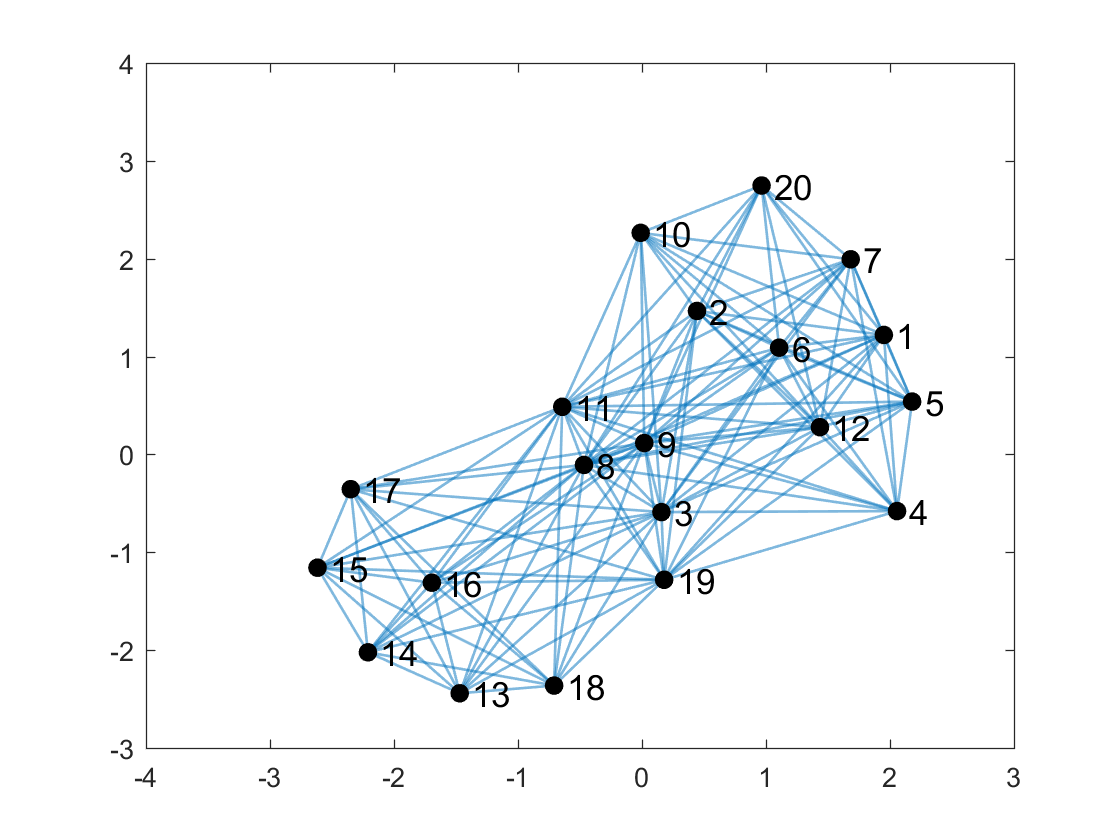} }}%
\hfill
\subfloat[GAMTL]{{\includegraphics[width=0.23\textwidth]{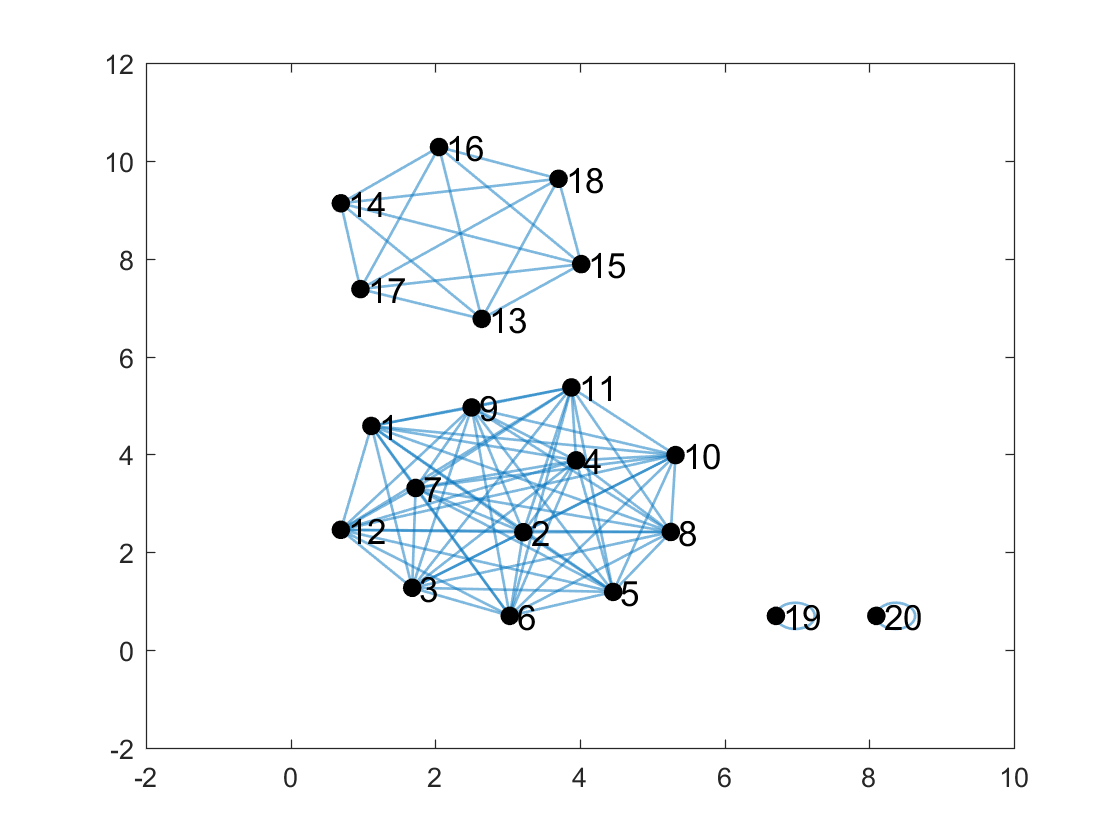} }}%
\hfill
\subfloat[RBF-GAMTL]{{\includegraphics[width=0.23\textwidth]{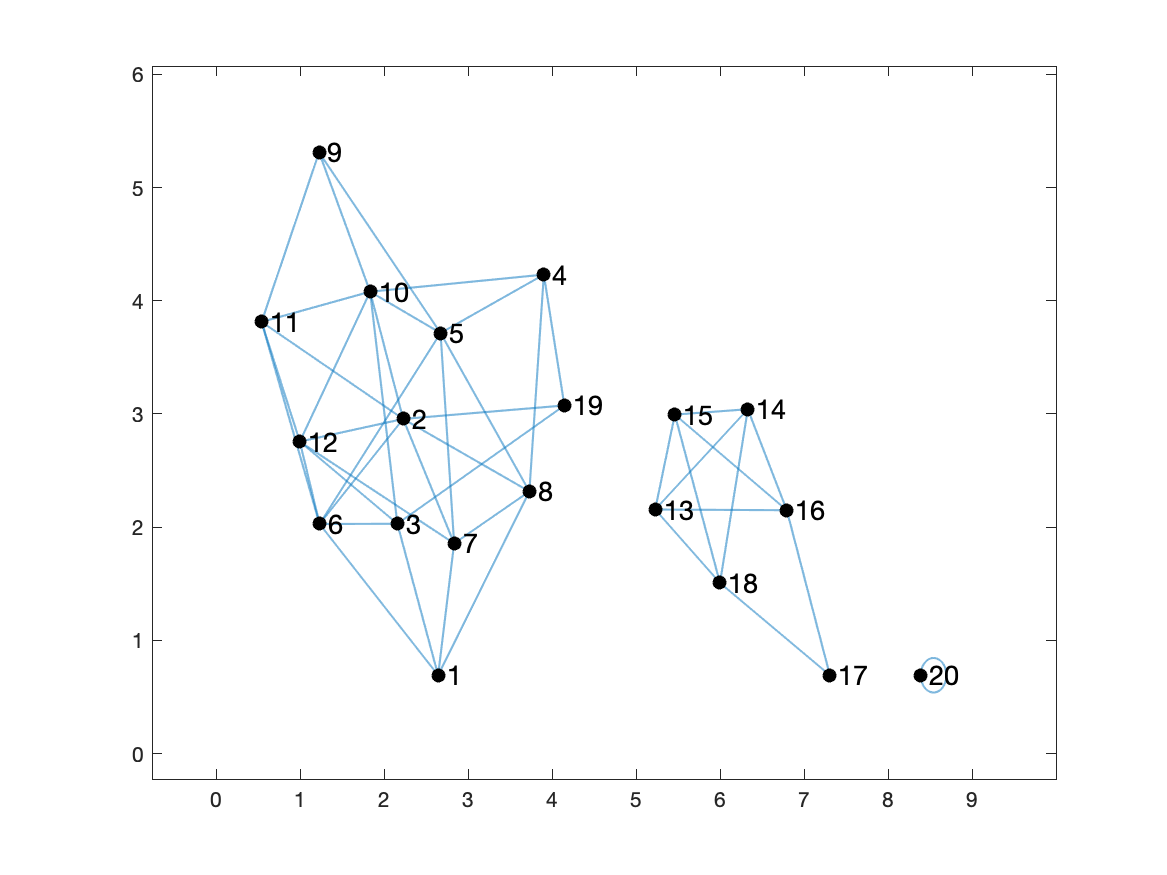}}}%
\caption{The task structure on $\tt{Syn}~1$ learned by (a) MTRL (task covariance matrix); (b) MSSL (inverse of task precision matrix); (c) BMSL (task covariance matrix); (d) TAT (tree structure, where tasks are leaf nodes); (e) GFL (correlation coefficient graph); (f) CCMTL ($k$-NN graph); (g) GAMTL (interpretable graph); and (h) RBF-GAMTL (interpretable graph). We use self-loop to underscore outliers.}
\label{fig:graph_topology_data1}
\end{figure*}

\begin{figure*}
\centering
\subfloat[MTRL]{{\includegraphics[width=0.23\textwidth]{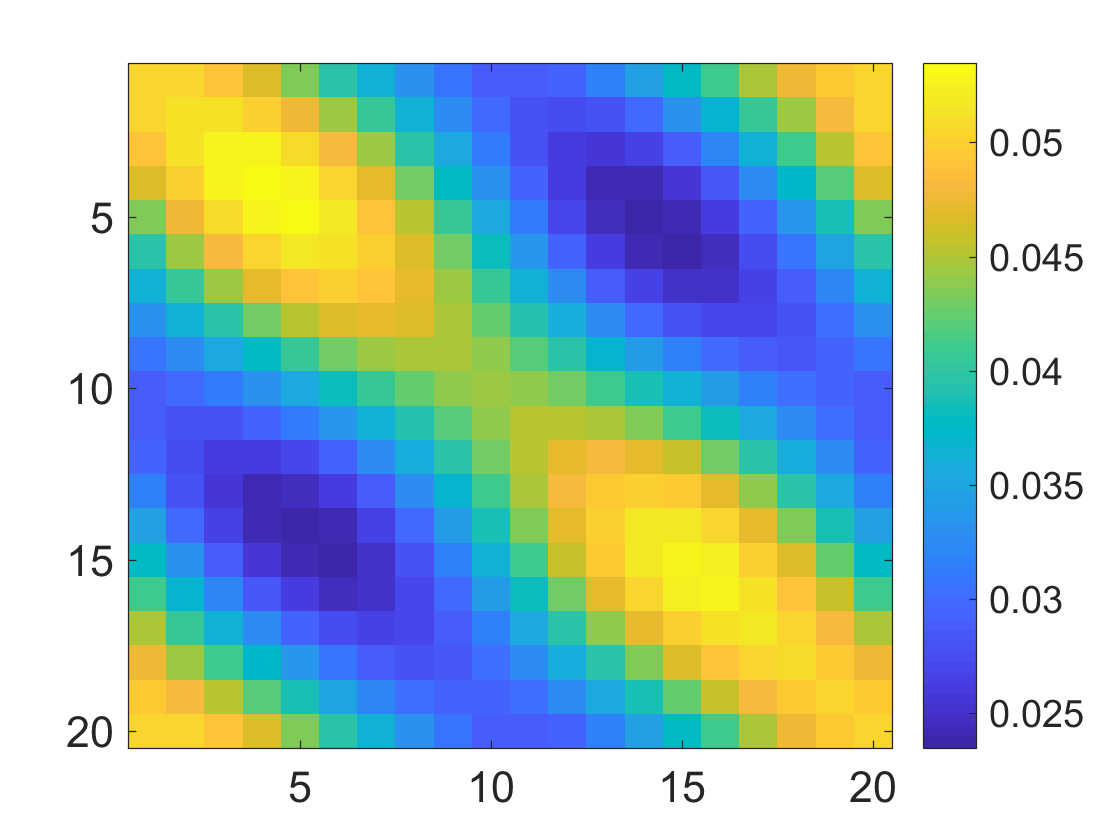} }}%
\hfill
\subfloat[MSSL]{{\includegraphics[width=0.23\textwidth]{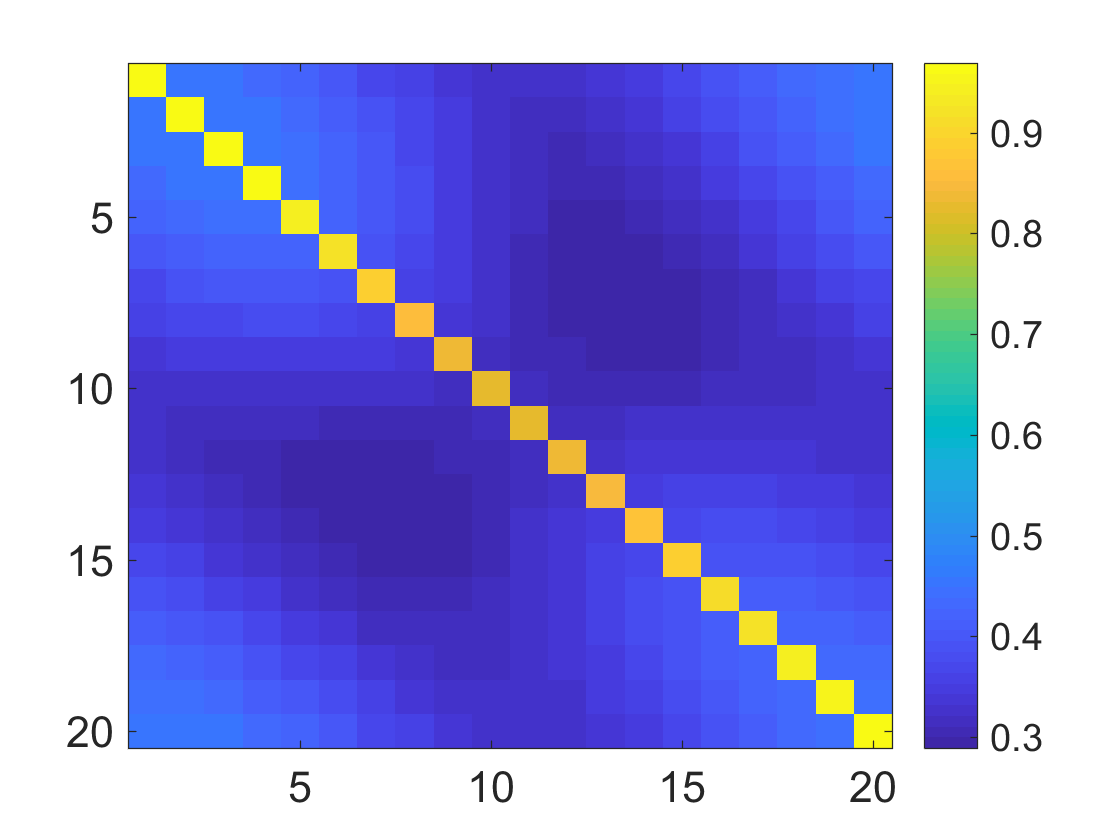} }}%
\hfill
\subfloat[BMSL]{{\includegraphics[width=0.23\textwidth]{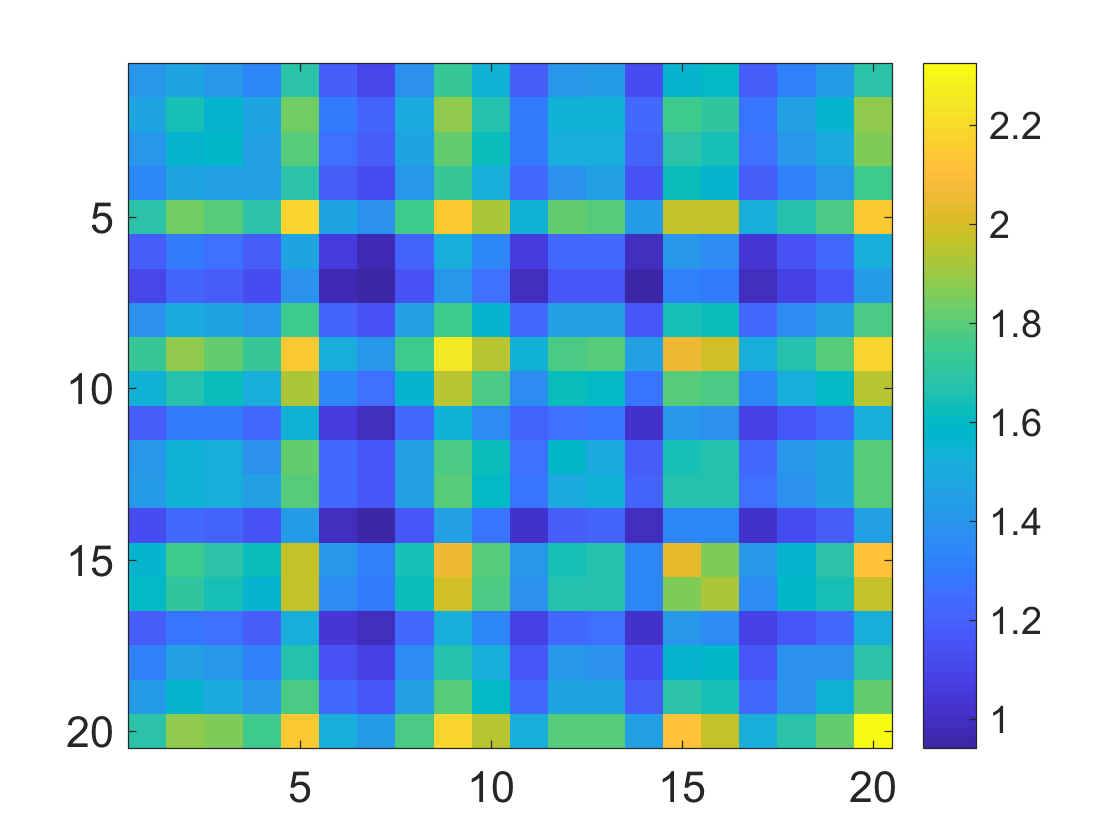} }}
\hfill
\subfloat[TAT]{{\includegraphics[width=0.18\textwidth]{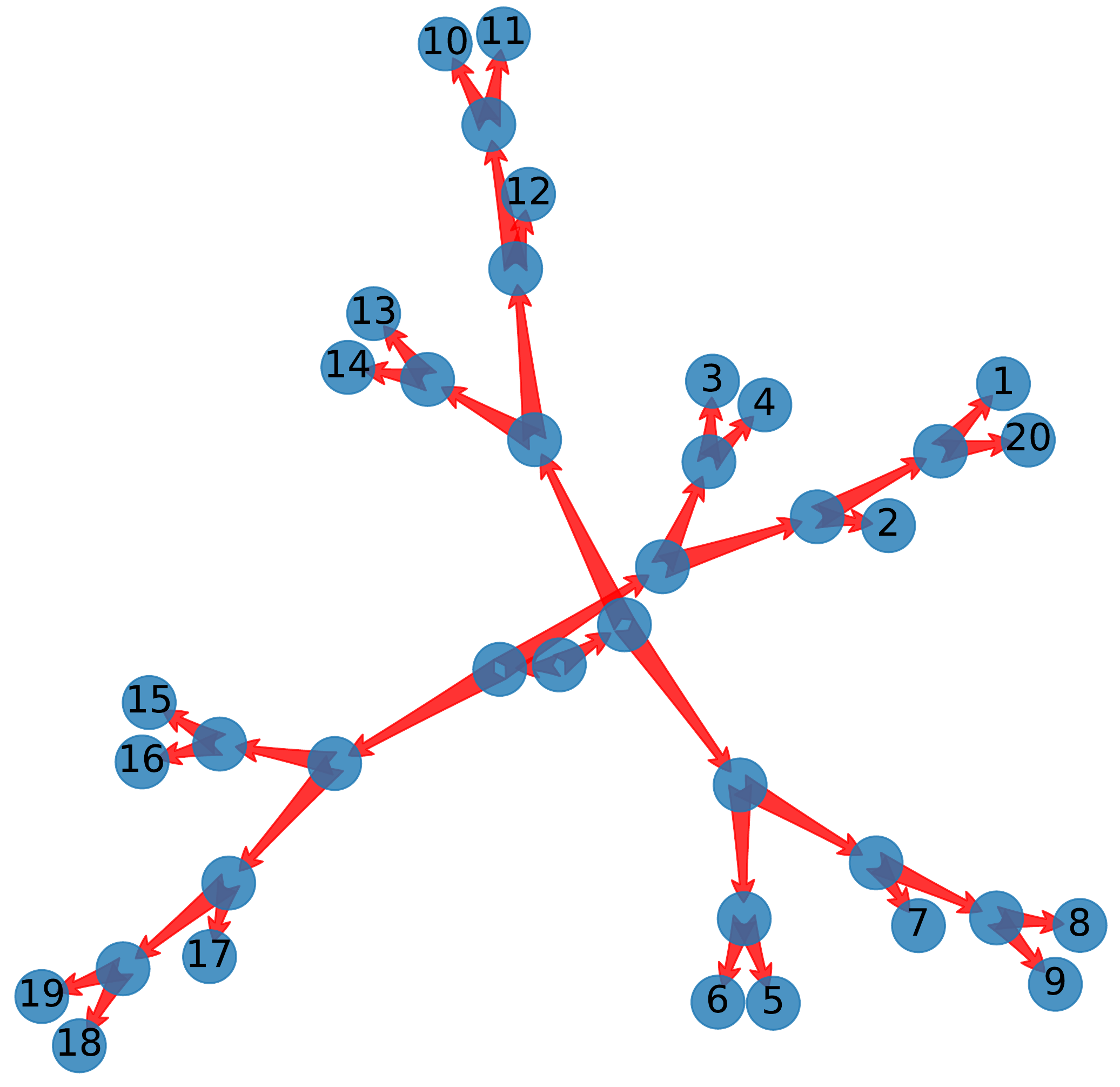} }}\\
\subfloat[GFL]{{\includegraphics[width=0.23\textwidth]{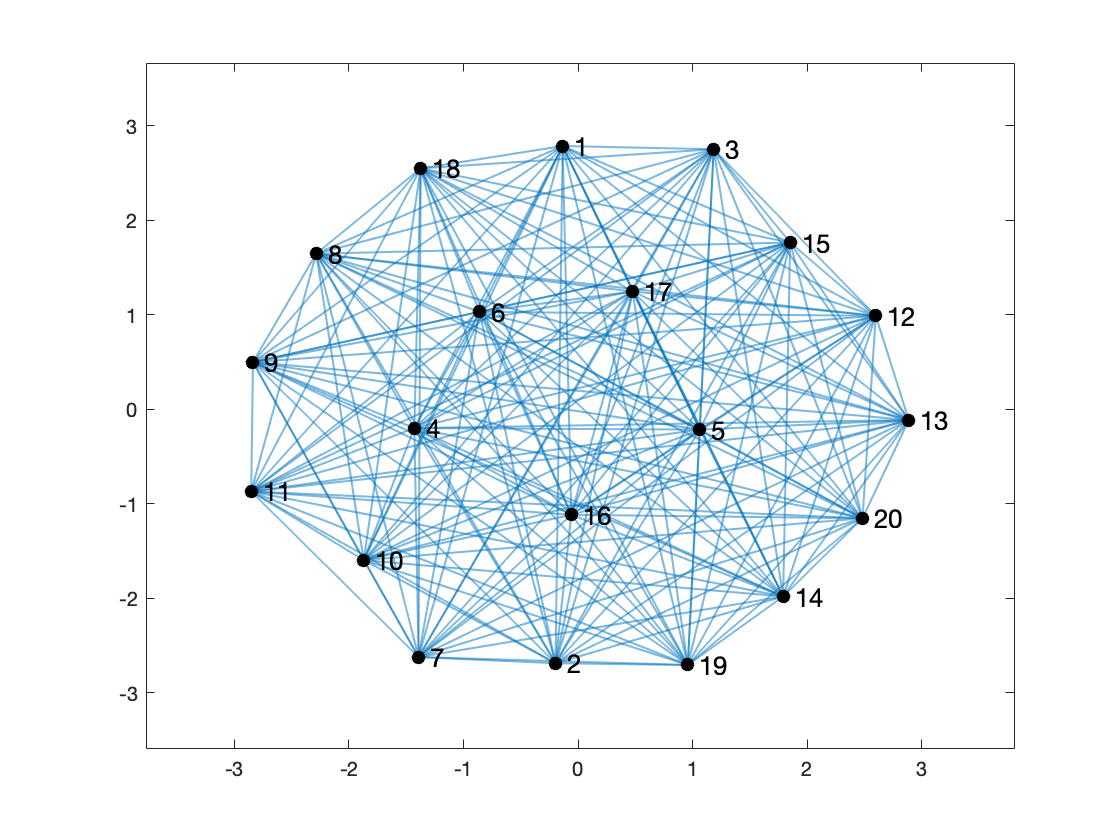} }}%
\hfill
\subfloat[CCMTL]{{\includegraphics[width=0.23\textwidth]{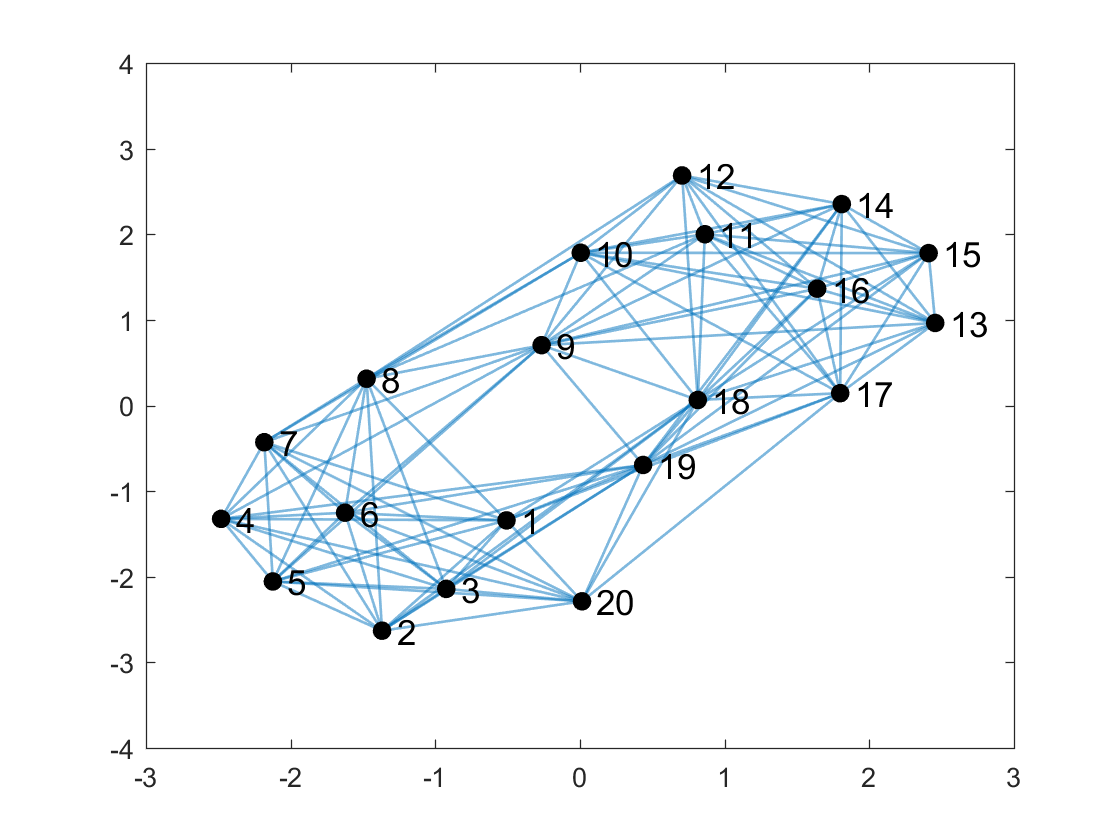} }}%
\hfill
\subfloat[GAMTL]{{\includegraphics[width=0.23\textwidth]{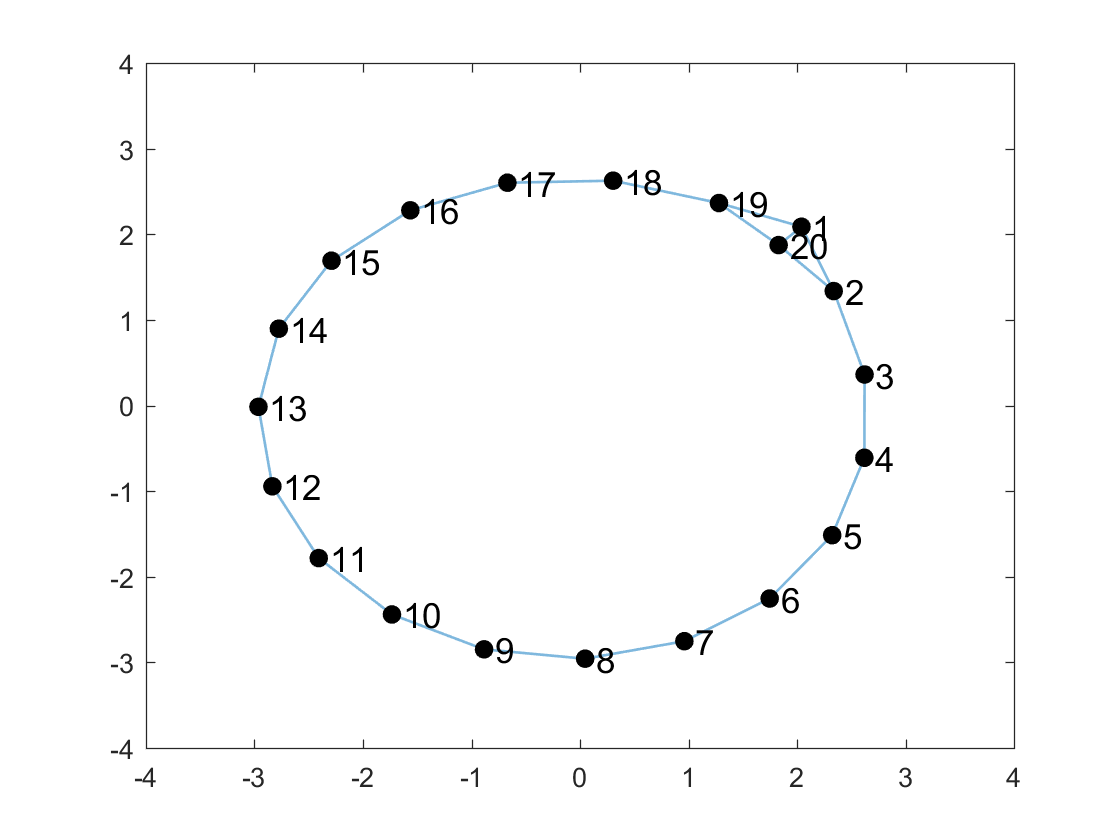} }}%
\hfill
\subfloat[RBF-GAMTL]{{\includegraphics[width=0.23\textwidth]{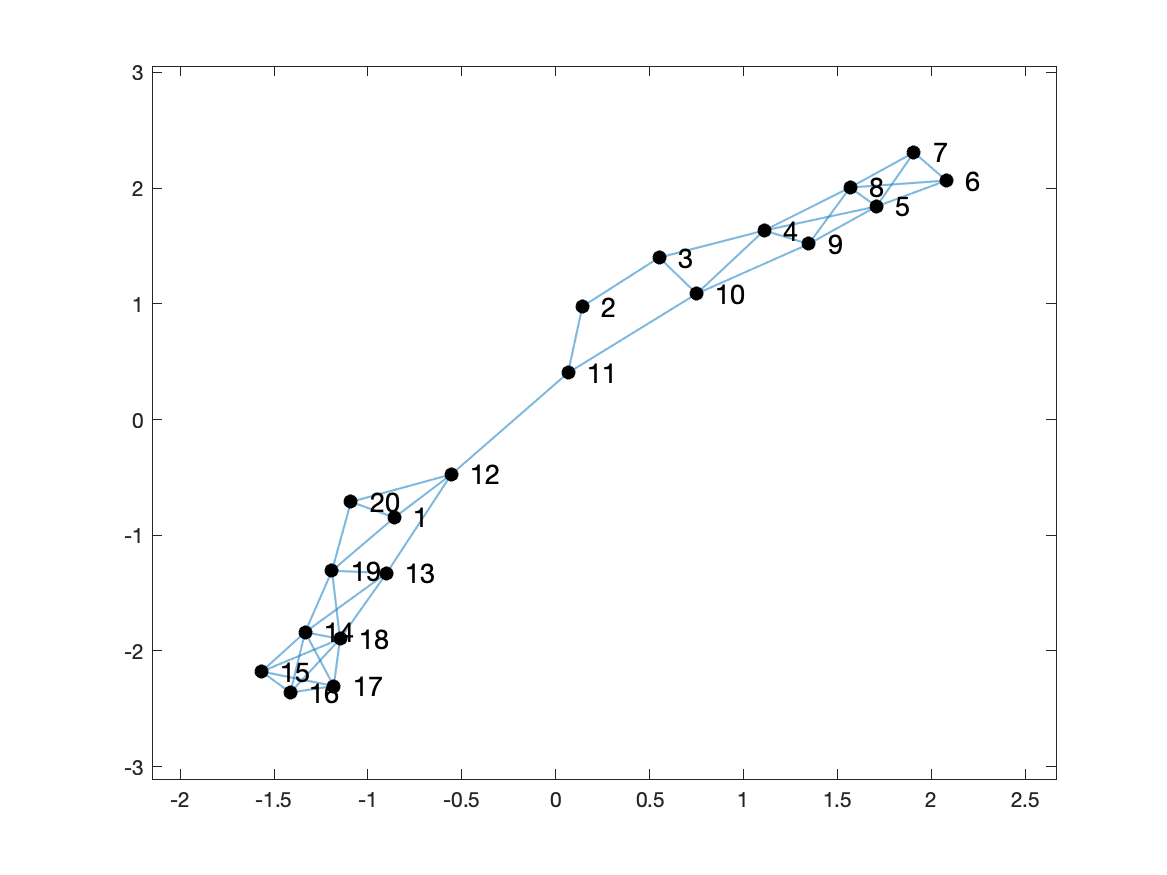} }}%
\caption{The task structure on $\tt{Syn}~2$ learned by (a) MTRL (task covariance matrix $\Omega$); (b) MSSL (inverse of task precision matrix); (c) BMSL (task covariance matrix); (d) TAT (tree structure, where tasks are leaf nodes); (e) GFL (correlation coefficient graph); (f) CCMTL ($k$-NN graph); (g) GAMTL (interpretable graph); and (h) RBF-GAMTL (interpretable graph).}
\label{fig:graph_topology_data2}
\end{figure*}

\subsection{Real-World Applications}

We then present three solid examples to demonstrate the utility and superiority of our GAMTL and RBF-GAMTL on real-world applications, involving bioinformatics, smart transportation, and signal and system. The performance of GFL is omitted in this section, mainly because GFL assumes that all tasks have the same input, which does not hold true in the general setup of multi-task learning.



\subsubsection{Parkinson's disease assessment}

This is a benchmark multi-task regression data set\footnote{\url{https://archive.ics.uci.edu/ml/datasets/parkinsons+telemonitoring}}, comprising a range of biomedical voice measurements taken from $42$ patients with early-stage Parkinson's disease. For each patient, the goal is to predict the motor Unified Parkinson's Disease Rating Scale (UPDRS) score based $18$-dimensional record: age, gender, and $16$ jitter and shimmer voice measurements. We treat UPDRS prediction for each patient as a task, resulting in $42$ tasks and $5,875$ observations in total.

The RMSE values of all competing methods with respect to different train/test ratios are summarized in Table~\ref{Tab:parkinson}. MTRL is unstable when training samples is less. GAMTL improves marginally over CCMTL, and is constantly superior to MMSL, BMSL and TAT. RBF-GAMTL significantly reduces the generalization error.

Before illustrating task structures generated by different methodologies, we first perform a preliminary study on the pairwise relatedness between any two tasks. To this end, let us suppose each task is represented by input $\mathbf{x}$ and output $y$, we model the relatedness between tasks $T_1$ and $T_2$ as the Kullback-Leibler (KL) divergence between their respective posterior distributions $p_1(y|\mathbf{x})$ and $p_2(y|\mathbf{x})$, i.e., $D_{\text{KL}}(p_1(y|\mathbf{x})||p_2(y|\mathbf{x}))$.
Intuitively, a small conditional divergence value indicates a strong relation and vise versa. We decompose $D_{\text{KL}}(p_1(y|\mathbf{x})||p_2(y|\mathbf{x}))$ by the Shannon's chain rule~\cite{mackay2003information} as $D_{\text{KL}}(p_1(\mathbf{x},y)||p_2(\mathbf{x},y))- D_{\text{KL}}(p_1(\mathbf{x})||p_2(\mathbf{x}))$ and estimate each term with an adaptive $k$NN estimator~\cite{wang2009divergence}. We project the generated conditional divergence matrix into a $3$d plane using multidimensional scaling (MDS) to form the graph coordinates.

The generated graphs by GAMTL and RBF-GAMTL (with $0.5$ train/test ratio) are plotted in Fig.~\ref{fig:parkinson_graph}(a) and \ref{fig:parkinson_graph}(b), respectively. In general, there is a close correspondence between two graphs and the conditional KL divergence: tasks with small divergences are likely to be grouped together and there is no abnormal connections between two tasks that are far away from each other. By contrast, a dense $k$-NN graph in CCMTL (see Fig.~\ref{fig:parkinson_graph}(c)) is hard to interpret, and the task relationship in MSSL (see Fig.~\ref{fig:parkinson_graph}(f)) is dominated by its diagonal (which suggests weak connections to other tasks). On the other hand, it is hard to discover useful patterns (such as outliers or groups of tasks) directly from the task covariance matrix generated by either BMSL (see Fig.~\ref{fig:parkinson_graph}(e)) or MTRL (see Fig.~\ref{fig:parkinson_graph}(g)). Moreover, if we look deeper, it seems that the generated matrix of BMSL or MTRL does not match well with conditional KL divergence. For example, tasks $30$ and $35$ are located away from most of other, GAMTL or RBF-GAMTL either identify them as outliers or only connect them with one or two edges, but both BMSL and MTRL suggest strong covariance between tasks $30$ and $35$ with other tasks.
Same as in the synthetic data, TAT is able to identify local relationships between tasks in the same subtree with a common parent node. However, it is hard for TAT to further quantitatively measure the global closeness between tasks from different subtrees.




\begin{figure*}[!hbpt]
\centering
\subfloat[GAMTL]{{\includegraphics[width=0.23\textwidth]{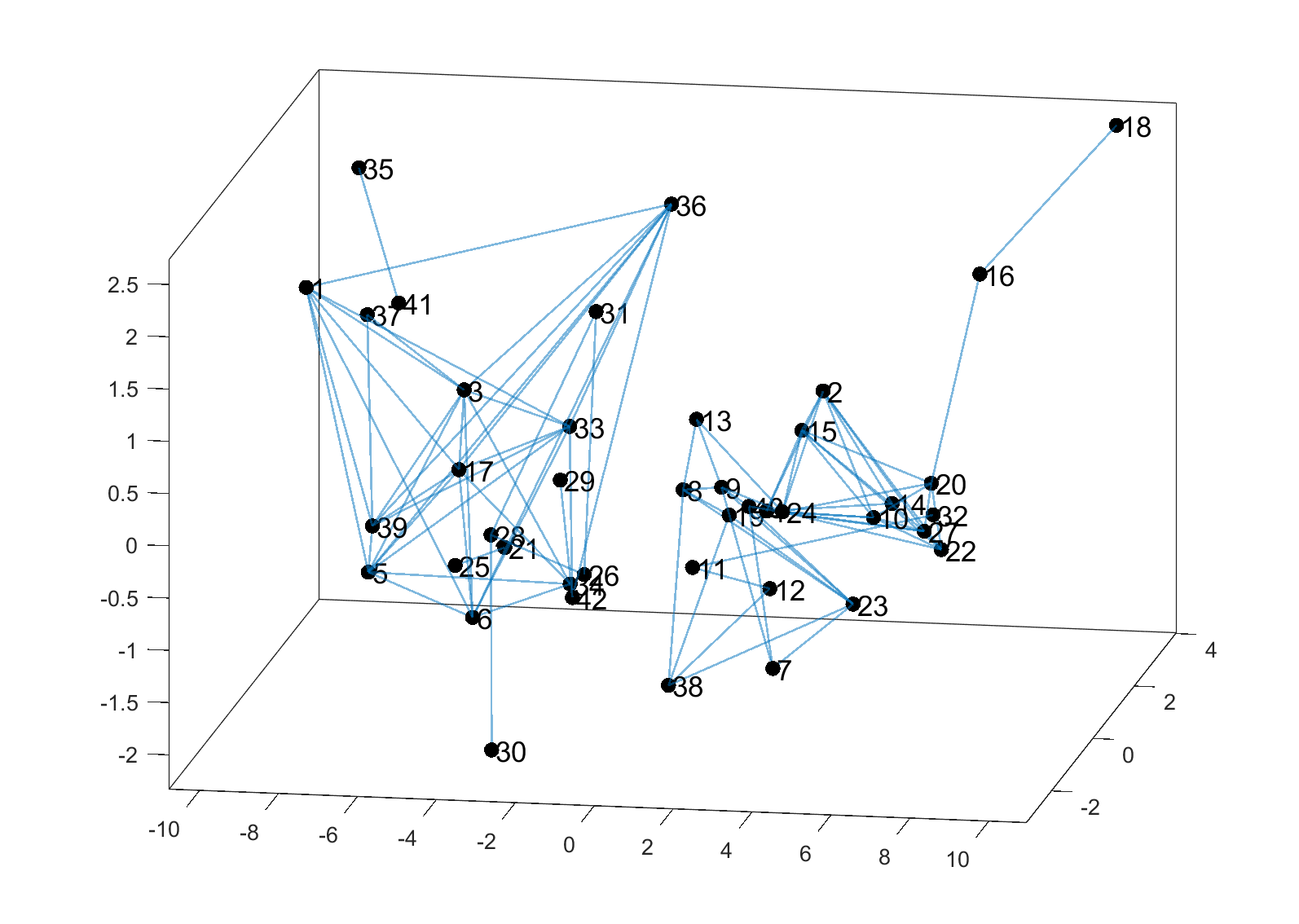}}}%
\hfill
\subfloat[RBF-GAMTL]{{\includegraphics[width=0.23\textwidth]{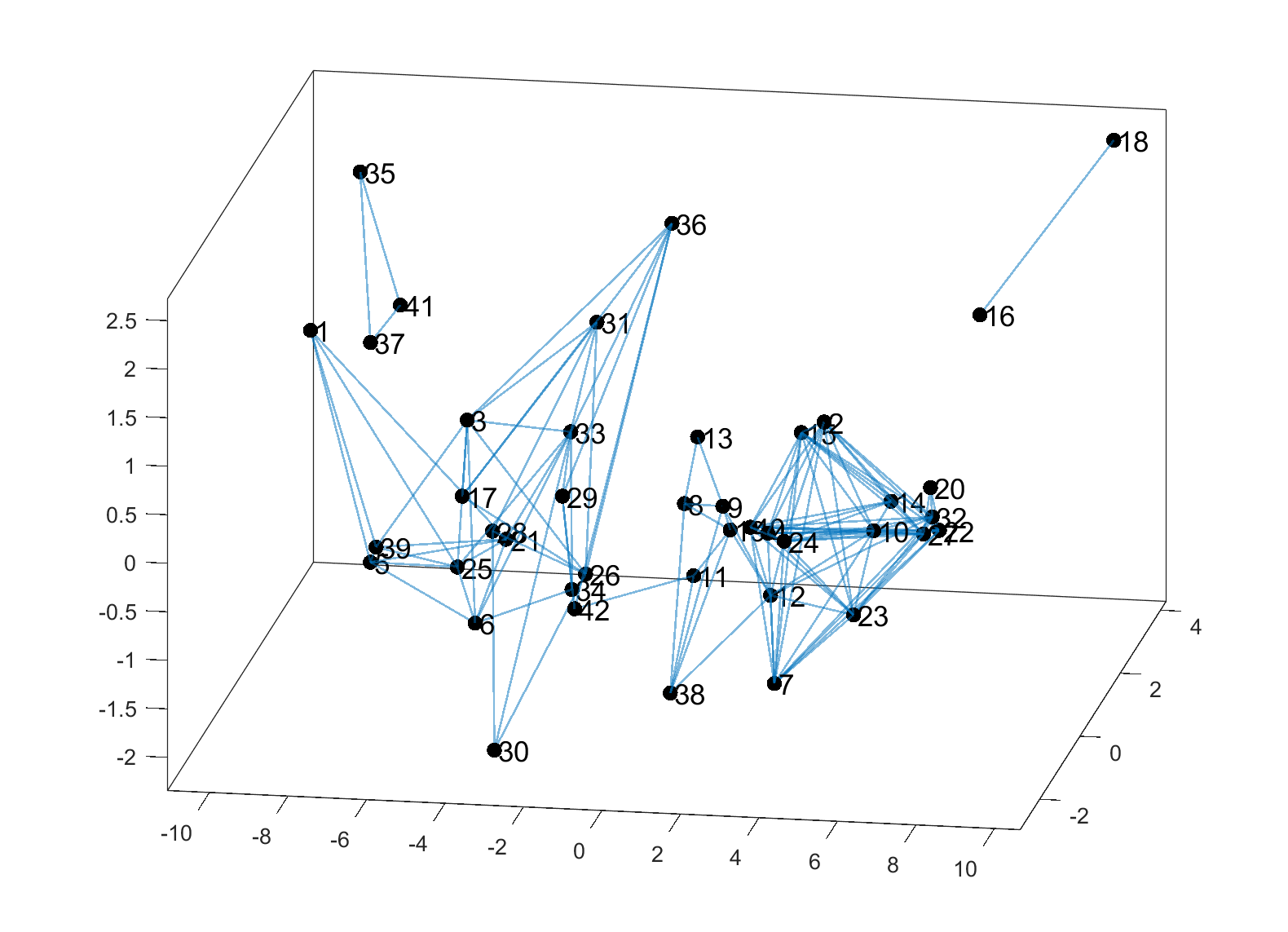}}}
\hfill
\subfloat[CCMTL]{{\includegraphics[width=0.23\textwidth]{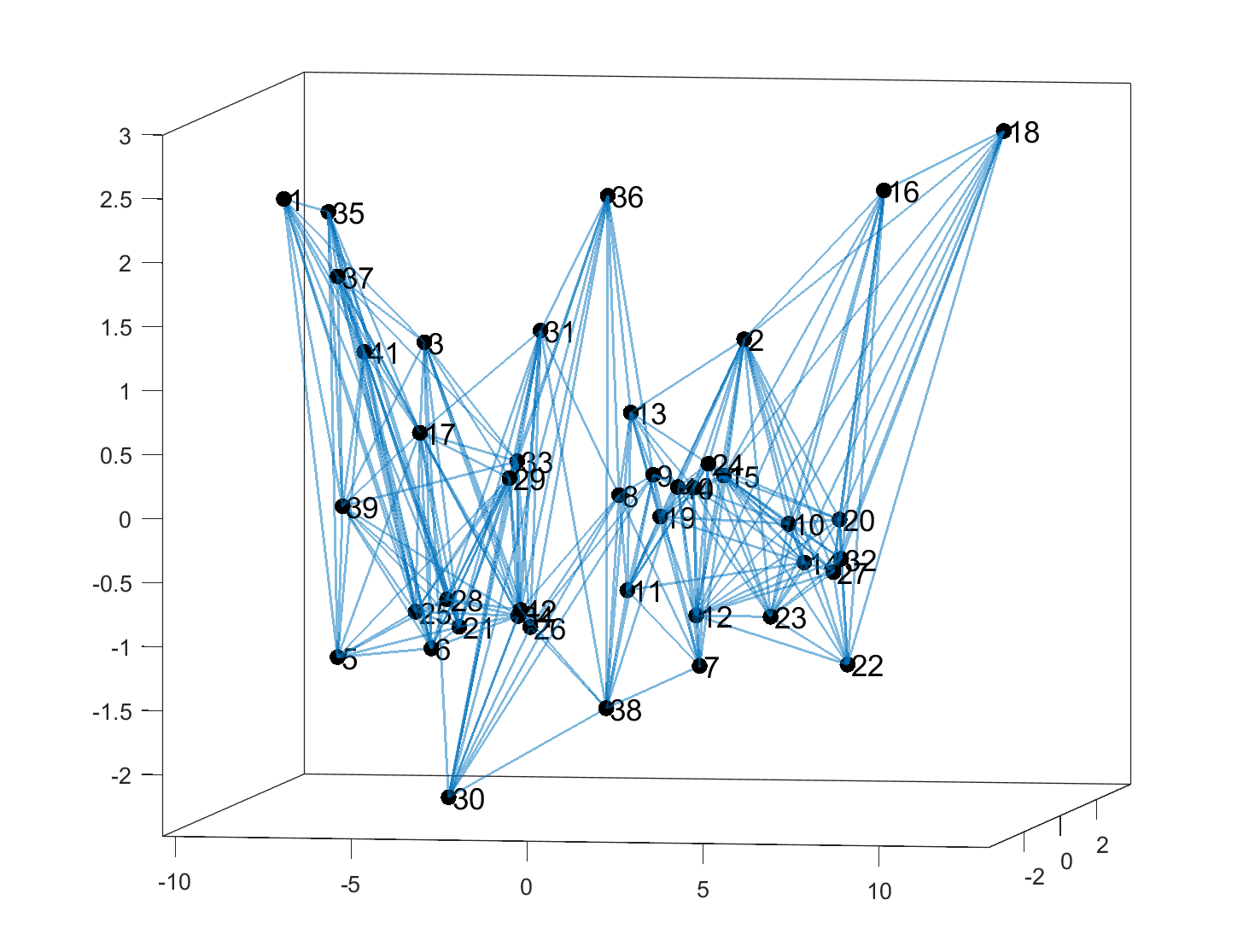} }} \\
\subfloat[TAT]{{\includegraphics[width=0.18\textwidth]{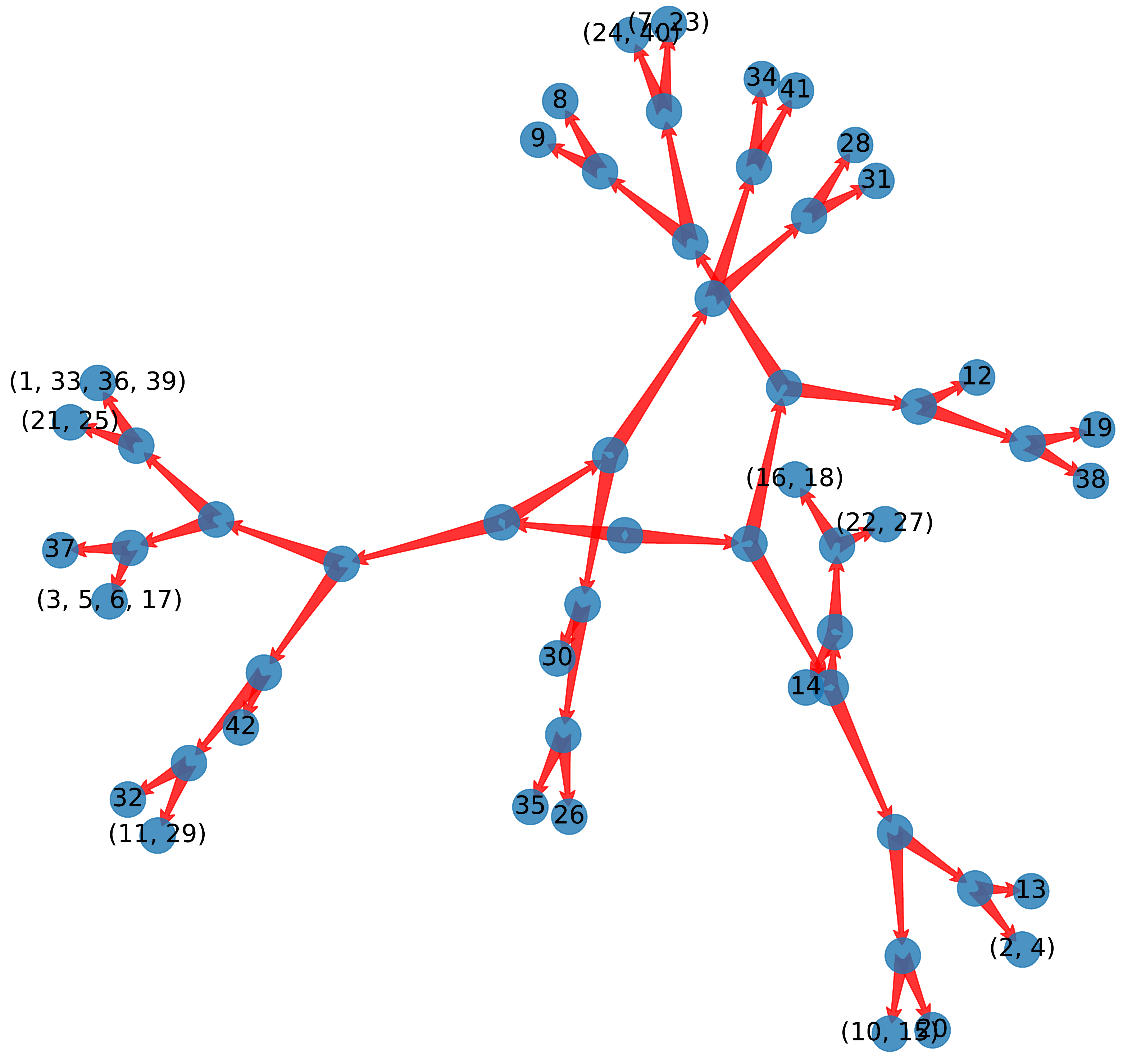}}}
\hfill
\subfloat[BMSL]{{\includegraphics[width=0.23\textwidth]{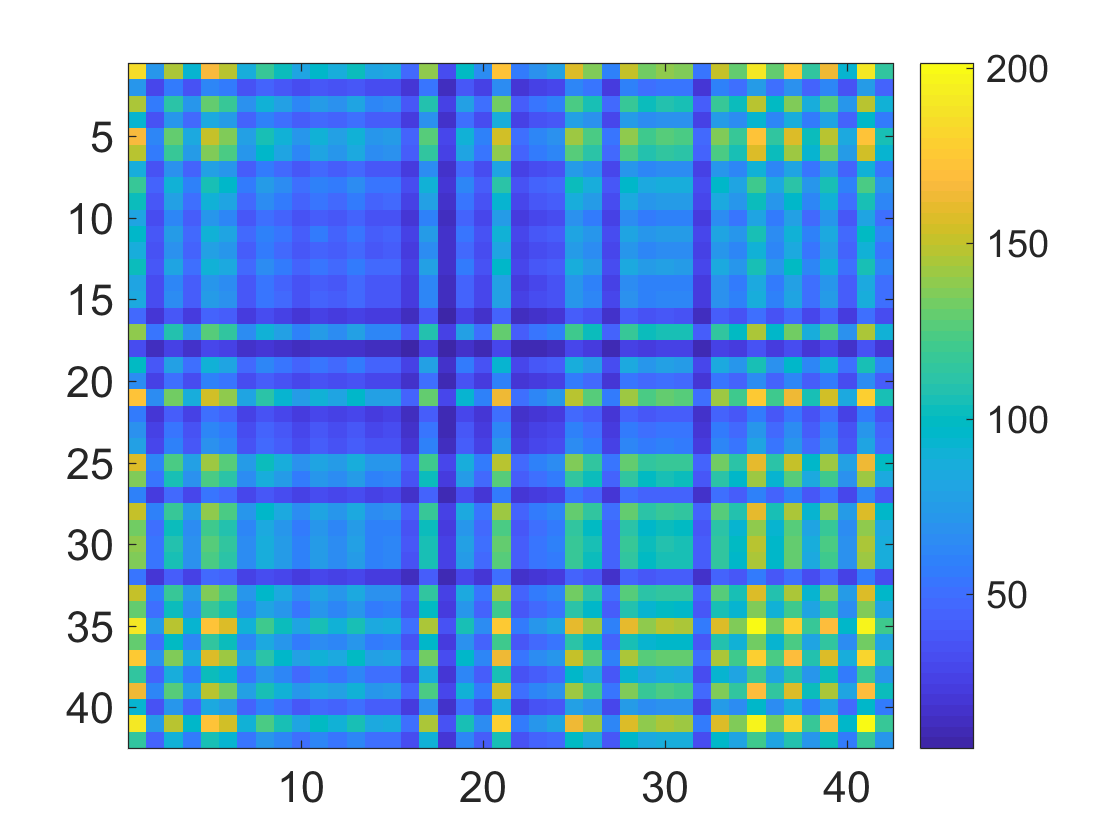} }}%
\hfill
\subfloat[MMSL]{{\includegraphics[width=0.23\textwidth]{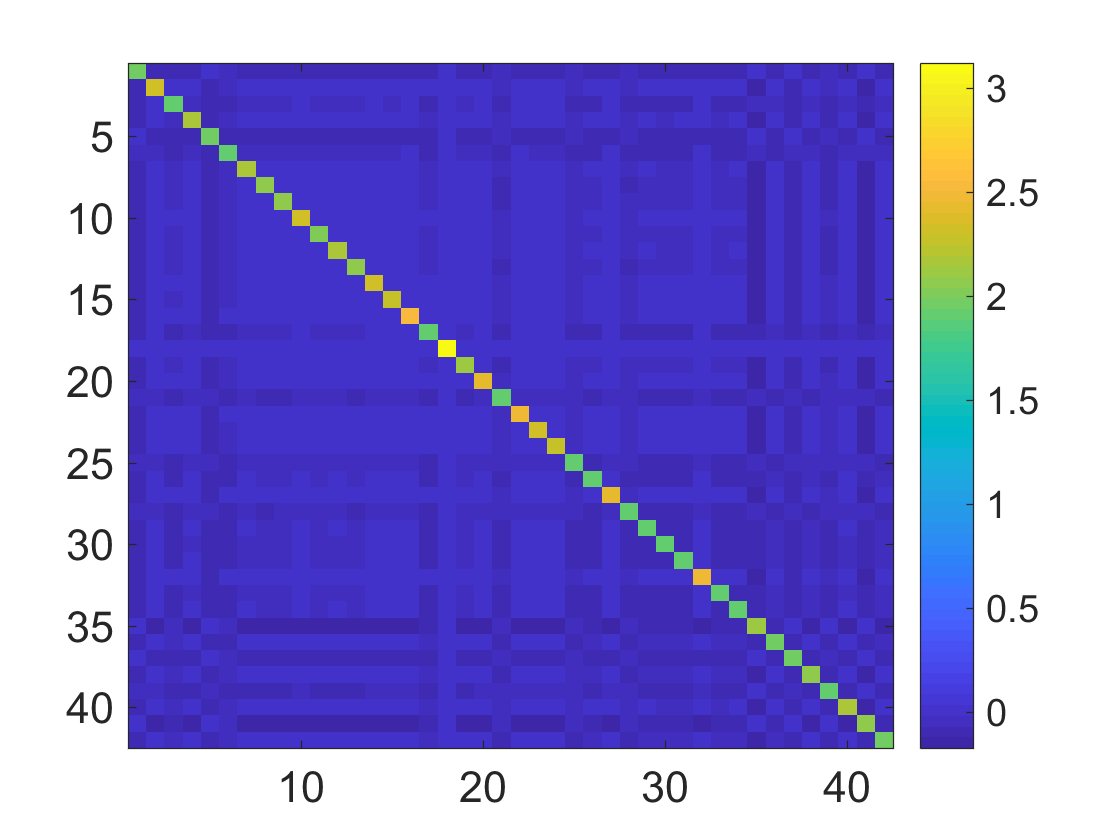} }}%
\hfill
\subfloat[MTRL]{{\includegraphics[width=0.23\textwidth]{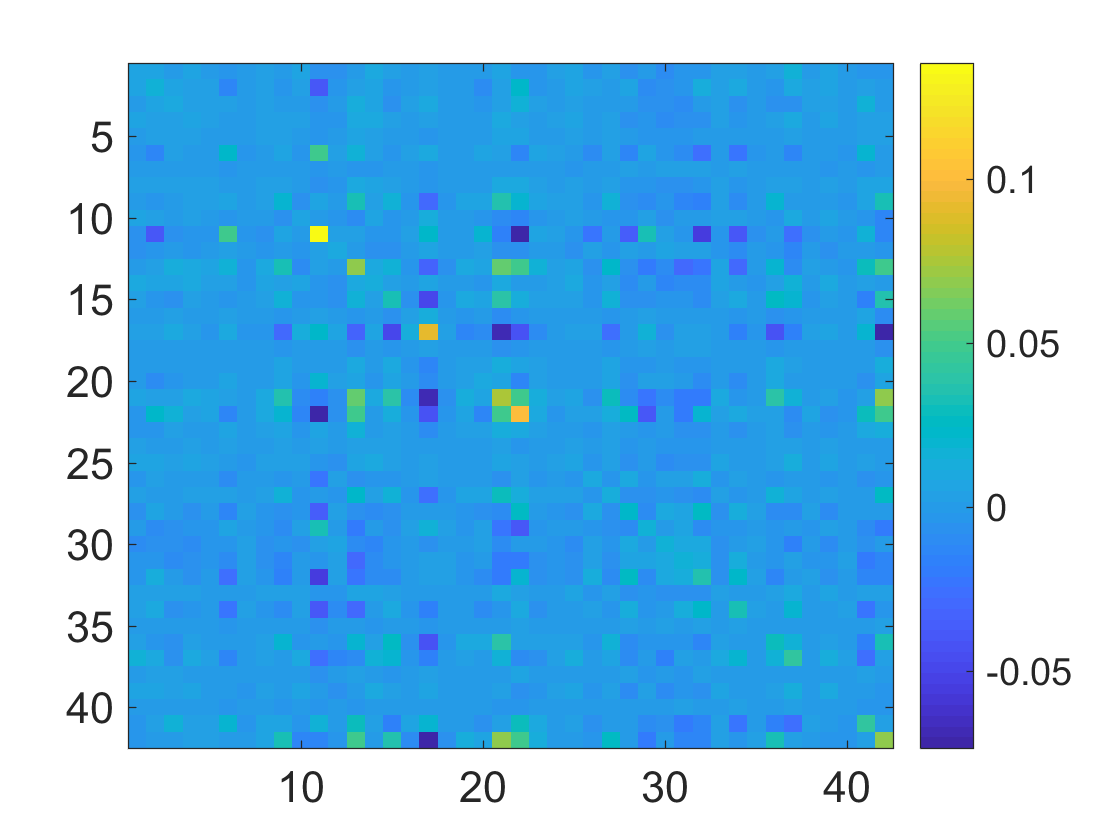} }}%
\caption{The task structure on Parkinson's disease data set learned by (a) GAMTL; (b) RBF-GAMTL; (c) CCMTL; (d) TAT; (e) BMSL; (f) MMSL; and (g) MTRL. Graph coordinates are generated by MDS over a dissimilarity matrix evaluated with (symmetric) conditional KL divergence.}
\label{fig:parkinson_graph}
\end{figure*}

\begin{table*}[!hbpt]
\centering
\small
\caption{RMSE (mean$\pm$std) on Parkinson's disease data set over $10$ independent runs with respect to different train/test ratios $r$. The best two performances are marked in bold and underlined, respectively.}\label{Tab:parkinson}
\begin{tabular}{cccccccc}
\toprule
 & MTRL & MSSL & BMSL & TAT & CCMTL & \textbf{GAMTL} & \textbf{RBF-GAMTL} \\
\midrule
$r=0.3$ & $4.147\pm3.038$ & $1.144\pm0.007$ & $1.221\pm0.110$ & $1.146\pm0.011 $ & $1.228\pm0.016$ & $\underline{1.121}\pm0.031$ & $\mathbf{0.609}\pm0.067$ \\
$r=0.4$ & $3.202\pm2.587$ & $1.129\pm0.011$ & $1.150\pm0.100$ & $ 1.130\pm0.010 $ & $1.149\pm0.013$ & $\underline{1.068}\pm0.009$ & $\mathbf{0.535}\pm0.060$ \\
$r=0.5$ & $1.761\pm0.850$  & $1.130\pm0.009$ & $1.110\pm0.085$ & $1.129\pm0.015$ & $1.115\pm0.011$ & $\underline{1.057}\pm0.013$ & $\mathbf{0.417}\pm0.046$ \\
$r=0.6$ & $1.045\pm0.050$ & $1.123\pm0.013$ & $1.068\pm0.036$ & $ 1.124\pm0.019 $ & $1.092\pm0.015$ & $\underline{1.037}\pm0.015$ & $\mathbf{0.367}\pm0.032$ \\
\bottomrule
\end{tabular}
\end{table*}

\subsubsection{Parking occupancy prediction in Birmingham, U.K.}
In the second application, we aim to simultaneously predict car parking occupancy rate ($0-100\%$) in multiple parking lots in the city of Birmingham in the U.K., and, at the same time, infer the spatial-temporal relationships across these parking lots. We treat the prediction task in each parking lot as an individual task. The raw data was published by the Birmingham City Council\footnote{\url{https://data.birmingham.gov.uk/dataset/birmingham-parking}} (BCC) under the Open Government License v$3.0$ and was updated every $30$ minutes from $8:00$ to $16:30$ (18 occupancy values per parking lot and day). Here, we use a cleaned data set in~\cite{stolfi2017predicting}, comprising valid occupancy rates of $29$ car parking lots operated by National Car Parks (NCP) from Oct. $4$\ts{th} $2016$ to Dec. $19$\ts{th} $2016$ ($11$ weeks). For each parking lot, we build the dataset by using the occupancy rates of previous $2$ hours (or $4$ hours) as input to predict the occupancy rate of $30$ minutes in advance, resulting in $35,456$ samples in total.

According to the raw data from BCC, the parking lots $8$, $9$ and $17$ have the same longitude and latitude (approximate to $13$ decimal places). Meanwhile, in order to gauge the quality of our generated graph, we apply the constraint Dynamic Time Warping (cDTW)~\cite{sakoe1978dynamic} on pairwise occupancy rate sequences from two parking lots to construct a dissimilarity matrix, and then apply MDS to project this dissimilarity matrix onto a 3d plane to form graph coordinates. In this sense, the grouped nodes in the graph also suggest the nearness from a time series clustering perspective.



We select the first week of observation (Oct. $4$\ts{th} to Oct. $10$\ts{th}) to train and left the remaining ten weeks for testing. The RMSE over $10$ repetitions are summarized Table~\ref{Tab:time_series}. The task structure learned by all competing methods are demonstrated in Figs.~\ref{fig:Birmingham_graph}. Obviously, GAMTL and RBF-GAMTL identify similar and highly interpretable patterns on task relatedness, which also has a close correspondence to cDTW. Again, MMSL is dominated by its diagonal matrix. BMSL and MTRL suffer from poor interpretability on their respective task covariance matrix. The generated tree from TAT can group locally similar tasks, but it does not identify critical global structures or provide quantitative measures on connections among tasks (especially for those which are originated from the same parent node but split in deeper layers of the tree).

\begin{figure*}
\centering
\subfloat[GAMTL]{{\includegraphics[width=0.23\textwidth]{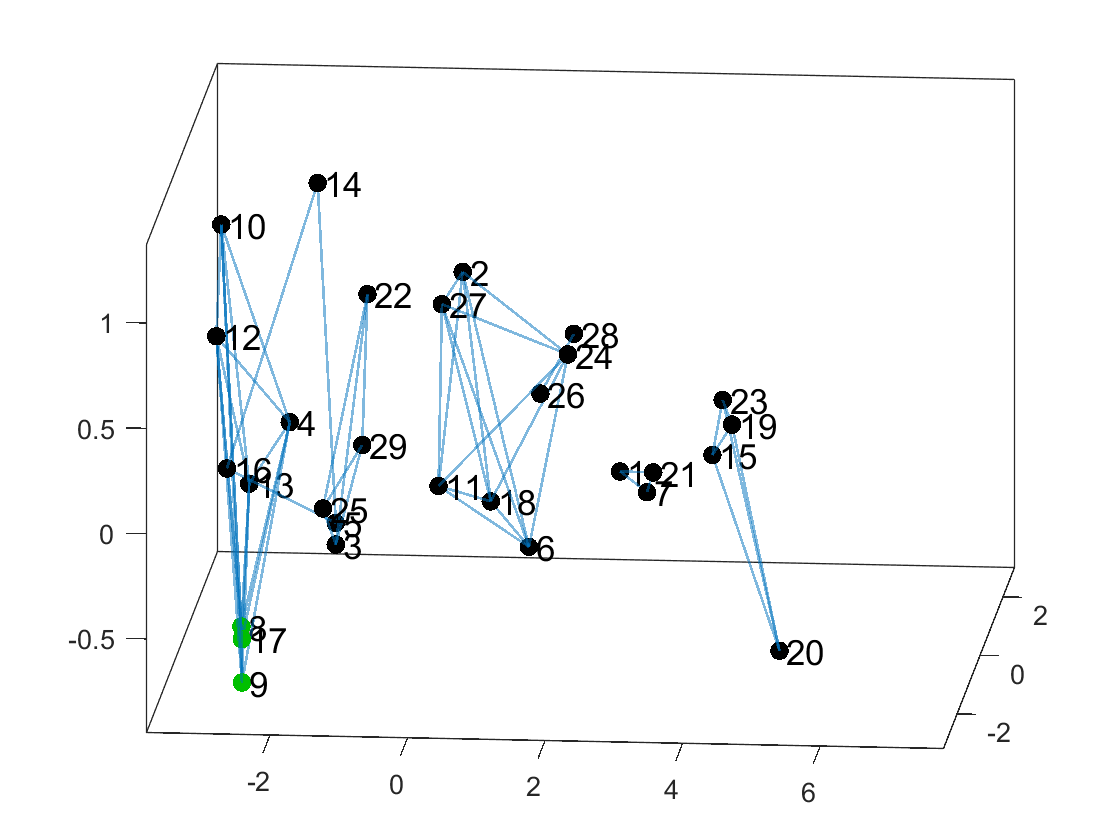} }}
\hfill
\subfloat[RBF-GAMTL]{{\includegraphics[width=0.23\textwidth]{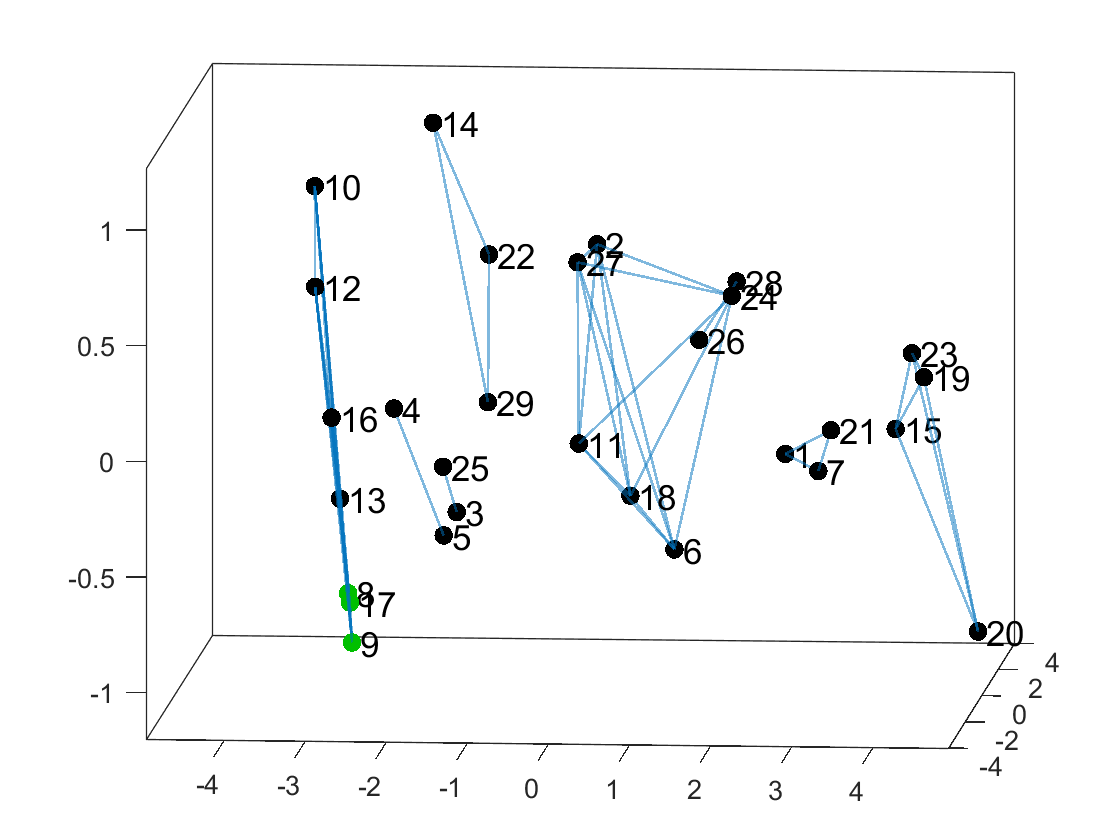} }}
\hfill
\subfloat[CCMTL]{{\includegraphics[width=0.23\textwidth]{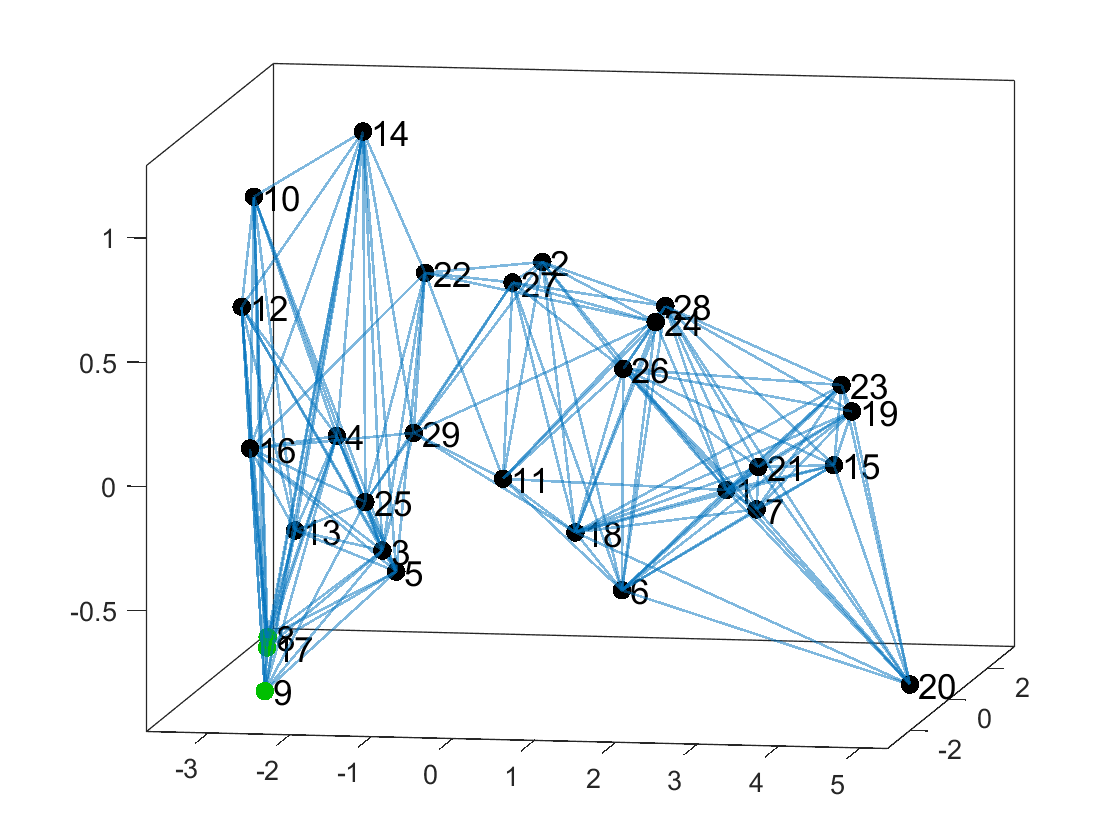} }} \\
\subfloat[TAT]{{\includegraphics[width=0.18\textwidth]{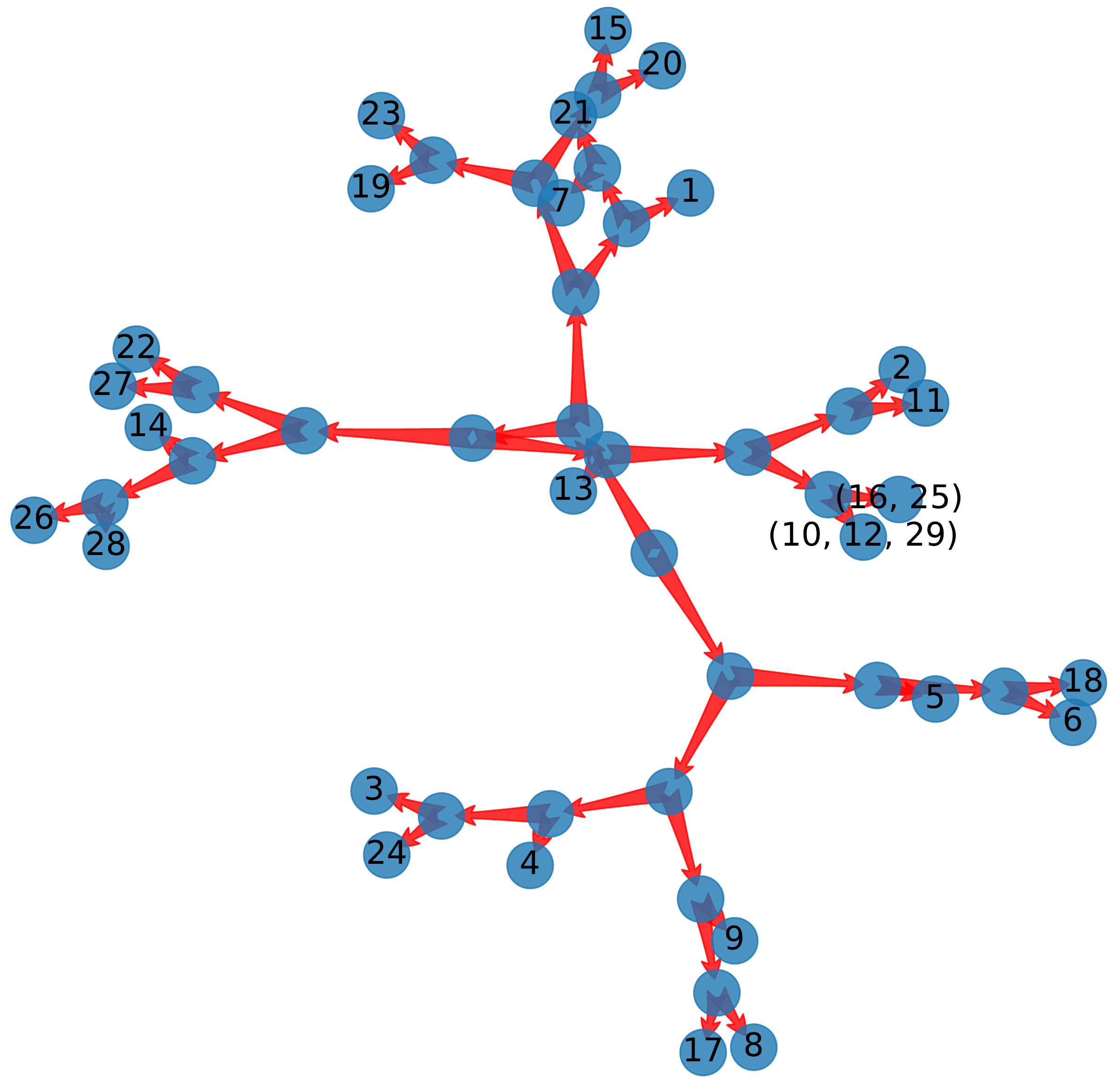} }}
\hfill
\subfloat[BMSL]{{\includegraphics[width=0.23\textwidth]{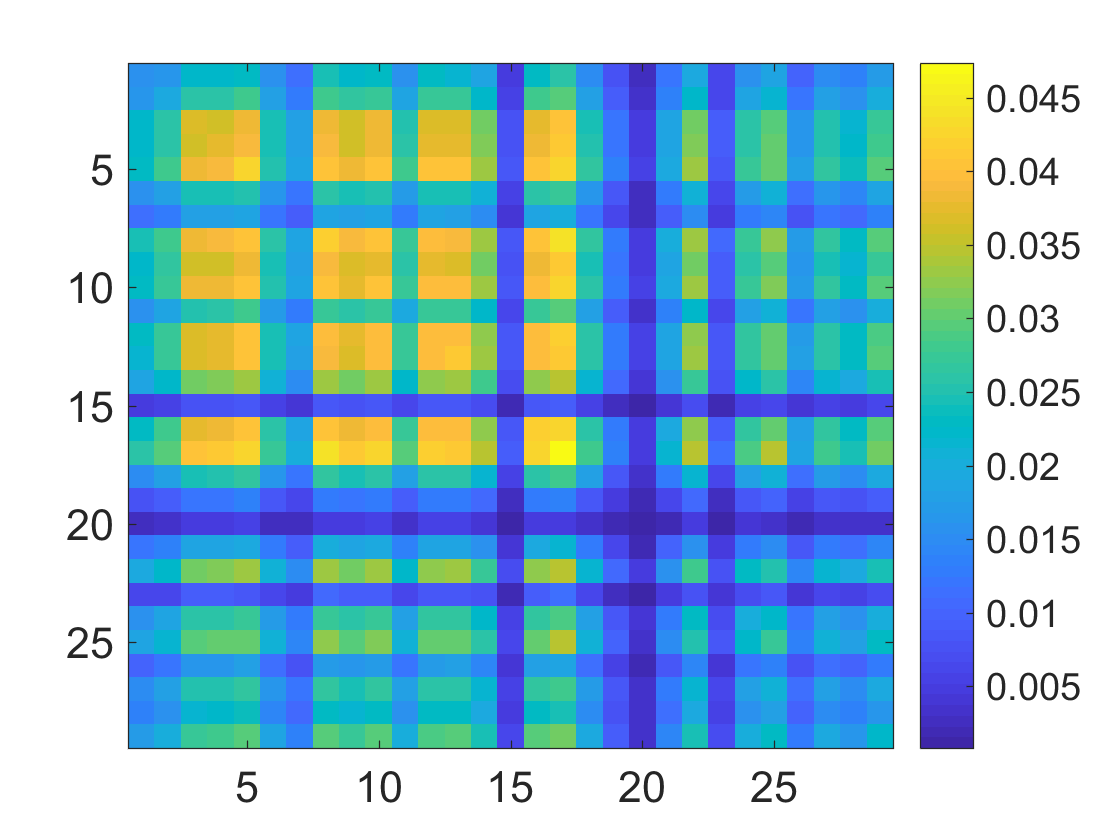} }}
\hfill
\subfloat[MMSL]{{\includegraphics[width=0.23\textwidth]{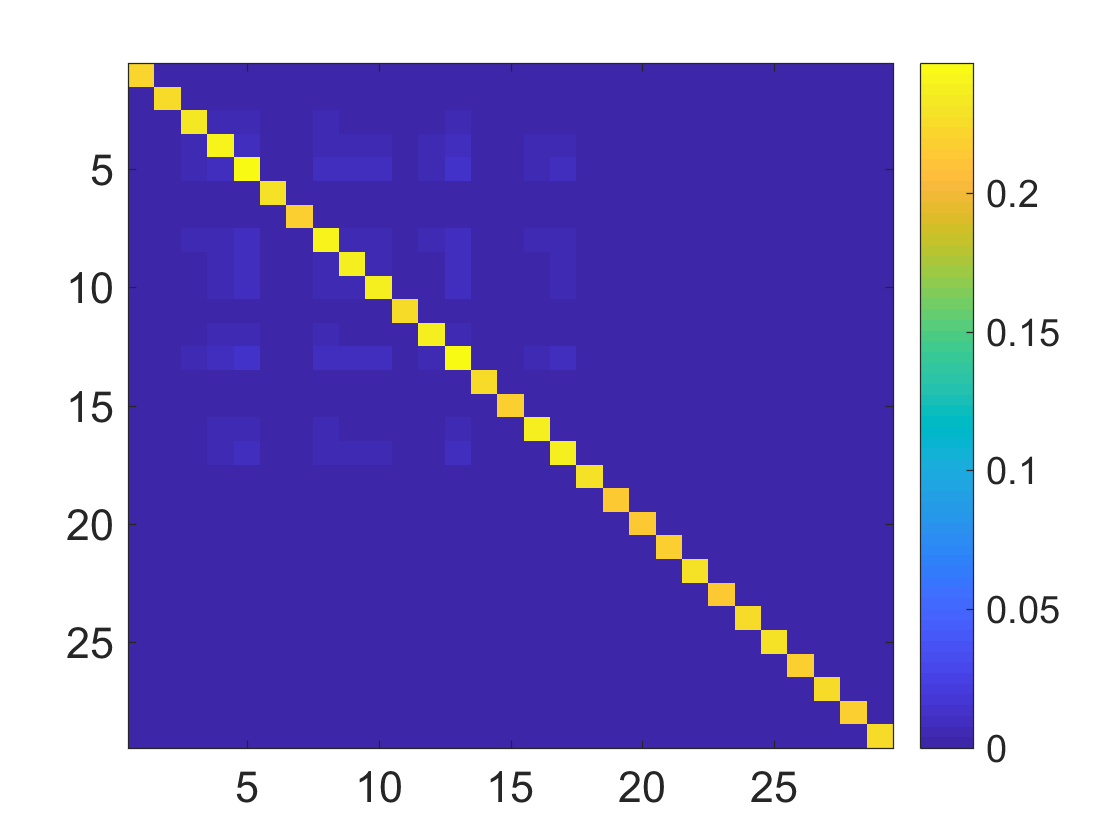} }}
\hfill
\subfloat[MTRL]{{\includegraphics[width=0.23\textwidth]{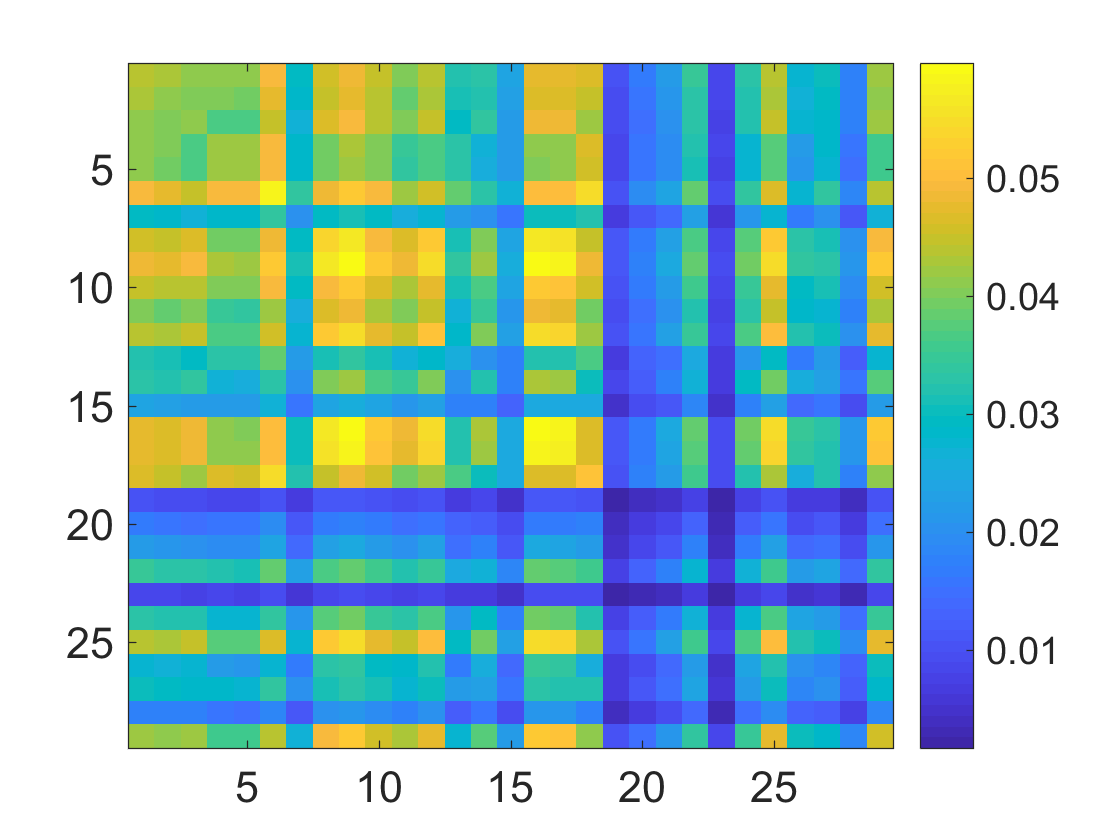} }}
\caption{The task structure on parking occupancy prediction in Birmingham learned by (a) GAMTL; (b) RBF-GAMTL; (c) CCMTL; (d) TAT; (e) BMSL; (f) MMSL; and (g) MTRL. Graph coordinates are generated by MDS over a dissimilarity matrix evaluated with cDTW. Nodes $8$, $9$, $17$ are marked with green.}
\label{fig:Birmingham_graph}
\end{figure*}

\subsection{Dynamical system identification over networks}
Our final application involves system identification, where the objective is to set up a suitable parameterized identification model and adjust the parameters of the model to optimize a performance function based on the error between the desired signal of the system and the identification model outputs~\cite{kumpati1990identification}. Here, we consider system identification in a distributed environment, such as wireless sensor networks (WSN). In this scenario, each agent receives measurements in a streaming fashion, and they are required to estimate either a common (usually nonlinear) model or different individual models due to spatial dependencies by alternating local computations and communications with their neighbors~\cite[Chapter~10]{comminiello2018adaptive}.

Following the experimental setup in recent literature on adaptive filtering (e.g.,~\cite{chen2014multitask,bouboulis2017online}), we evaluate our methodology on a simulation platform. Specifically, the input signal at each node $k$ and time instant $i$ was a sequence of statistically independently $2$d vector defined as:
\begin{equation}
\mathbf{x}_{k,i} = [x_{k,i}(1),x_{k,i}(2)]^T,
\end{equation}
with correlated samples satisfying $x_{k,i}(1) = 0.5x_{k,i}(2) + v_{k,i}$. The second entry of $\mathbf{x}_{k,i}$ and $v_{k,i}$ were both $i.i.d.$ zero-mean Gaussian samples with variance $\sigma^2_{x,k}$ and $(1-\rho^2)\sigma^2_{x,k}$ ($\rho=0.5$ in this work), respectively.

\begin{table*}
\centering
\caption{RMSE on $\tt{Birmingham}$. The best two performances are marked in bold and underlined, respectively.}\label{Tab:time_series}
\begin{tabular}{cccccccc}
\toprule
 & MTRL & MSSL & BMSL & TAT & CCMTL & \textbf{GAMTL} & \textbf{RBF-GAMTL} \\
\midrule
$\tt{Birmingham}$ ($2$h embedding) & $0.0976$ & $0.0950$ & $0.0860$ & $0.0923$ & $0.0857$ & $\underline{0.0853}$ & $\mathbf{0.0730}$ \\
$\tt{Birmingham}$ ($4$h embedding) & $0.0970$ & $0.1120$ & $0.0880$ & $0.0903$ & $0.0842$ & $\underline{0.0838}$ & $\mathbf{0.0775}$ \\
\bottomrule
\end{tabular}
\end{table*}

The nonlinear system to be identified was the Wiener model:
\begin{equation}
  \psi(y_{k,i})=\begin{cases}
    \frac{y_{k,i}}{3[0.1+0.9y^2_{k,i}]^{1/2}} & \text{for $y_{k,i}\geq 0$}\\
    \frac{-y^2_{k,i}[1-\exp(0.7y_{k,i})]}{3} & \text{for $y_{k,i}<0$},
  \end{cases}
\end{equation}
where $y_{k,i} = \mathbf{w}^T\mathbf{x}_{k,i} - 0.2y_{k,i-1} + 0.35y_{k,i-2}$ is a embedded linear system with $2$nd-order memory and $\mathbf{w}$ is a $2$d linear coefficient vector on input $\mathbf{x}_{k,i}$. Additionally, we assume the system output $\psi(y)$ is corrupted by additive zero-mean Gaussian $i.i.d.$ noise $z$ with variance $\sigma^2_{z,k}$, i.e., $d_{k,i} = \psi(y_{k,i}) + z_{k,i}$.

We consider a network consisting of $10$ nodes with the topology depicted in Fig.~\ref{fig:system_graph}(a). $\sigma^2_x$ is uniformly distributed between $[0.005 0.015]$, whereas $\sigma^2_z$ is uniformly distributed between $[0.0005 0.0015]$.
The nodes were divided into $4$ clusters: $\mathcal{C}_1 = \{1,2,3\}$, $\mathcal{C}_2 = \{4,5,6\}$, $\mathcal{C}_3 = \{7,8\}$ and $\mathcal{C}_4 = \{9,10\}$. The $2$d linear coefficient vector $\mathbf{w}$ of the form $\mathbf{w}_{\mathcal{C}_i} = \mathbf{w}_0 + \delta \mathbf{w}_{\mathcal{C}_i}$ were chosen as $\mathbf{w}_0 = [0.5,-0.4]^T$, $\delta\mathbf{w}_{\mathcal{C}_1} = [0.2,-0.1]^T$, $\delta\mathbf{w}_{\mathcal{C}_2} = [0.2,0.1]^T$, $\delta\mathbf{w}_{\mathcal{C}_3} = [-0.3,0.1]^T$, and $\delta\mathbf{w}_{\mathcal{C}_4} = [0,0.1]^T$. We concatenate the coefficient vector in each agent to from a coefficient matrix $W\in\mathbb{R}^{2\times 10}$. The final coefficient matrix $W^\star$ (by taking into account inter-cluster communications) is given by $W^\star = WA$, where the mixing matrix $A$ was chosen according to the Metropolis rule:
\begin{equation}
  A_{k,l}=\begin{cases}
    \frac{1}{\max\{|\mathcal{N}_k|,|\mathcal{N}_l|\}} & \text{if $l\in\mathcal{N}_k$ and $l\neq k$}\\
    1-\sum_{i\in\mathcal{N}_k\setminus k}A_{k,i} & \text{$l=k$} \\
    0 & \text{otherwise}.
  \end{cases}
\end{equation}

The qualitative and quantitative evaluations are summarized in Figs.~(\ref{fig:system_graph}) and (\ref{fig:System_info}), respectively. For this complex nonlinear data set, it seems that only RBF-GAMTL can capture the underlying task relatedness and made correct predictions on system output. For TAT, it correctly discovered the closeness between agents $7$ and $8$, and between agents $9$ and $10$. However, it completely confused the relationship between agents $4$, $5$ and $6$. Moreover, it is obvious that linear models cannot model highly nonlinear mappings.

\begin{figure*}[!htbp]
\setlength{\abovecaptionskip}{0pt}
\setlength{\belowcaptionskip}{-10pt}
\centering
\subfloat[Network topology]{{\includegraphics[width=0.18\textwidth]{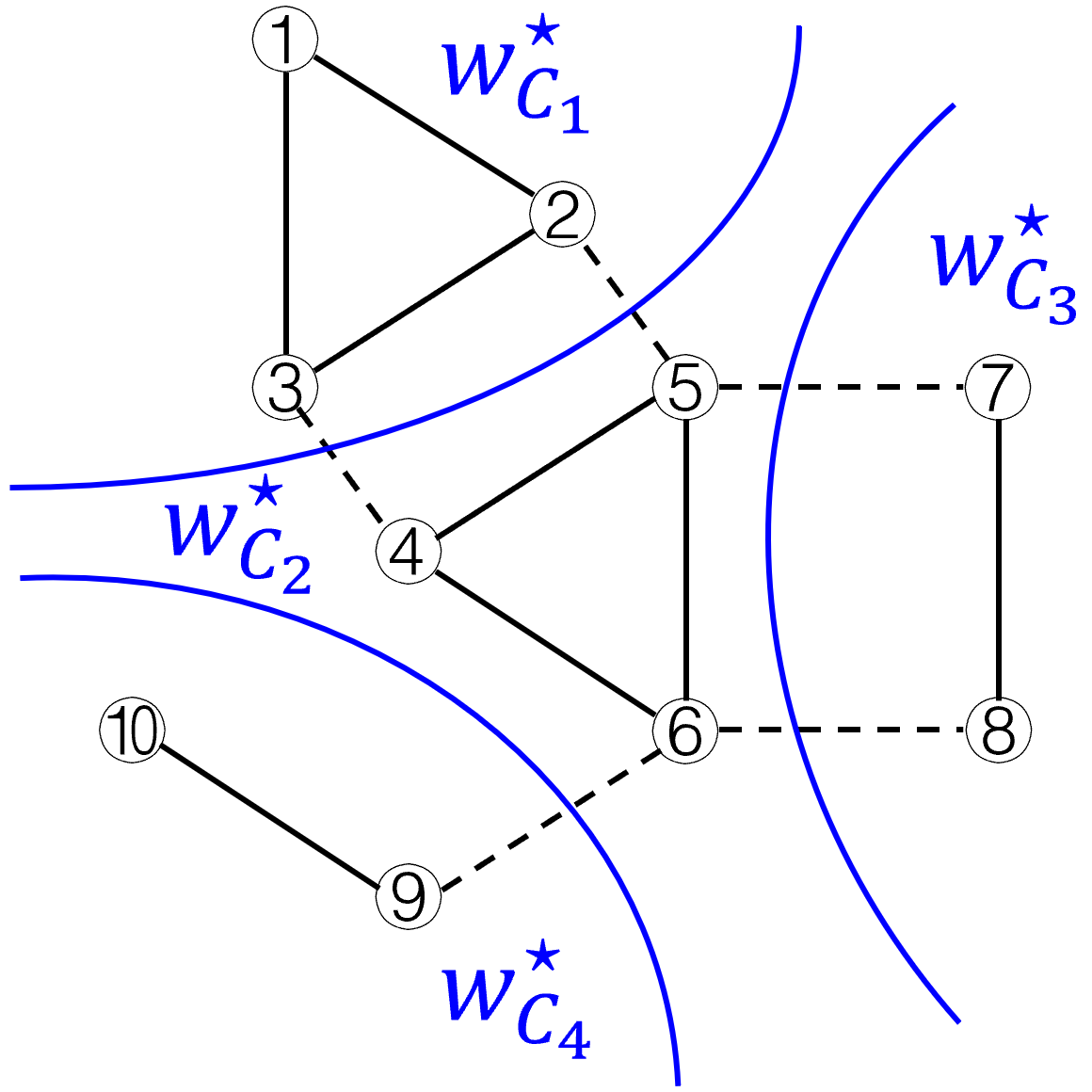}}}%
\hfill
\subfloat[GAMTL]{{\includegraphics[width=0.23\textwidth]{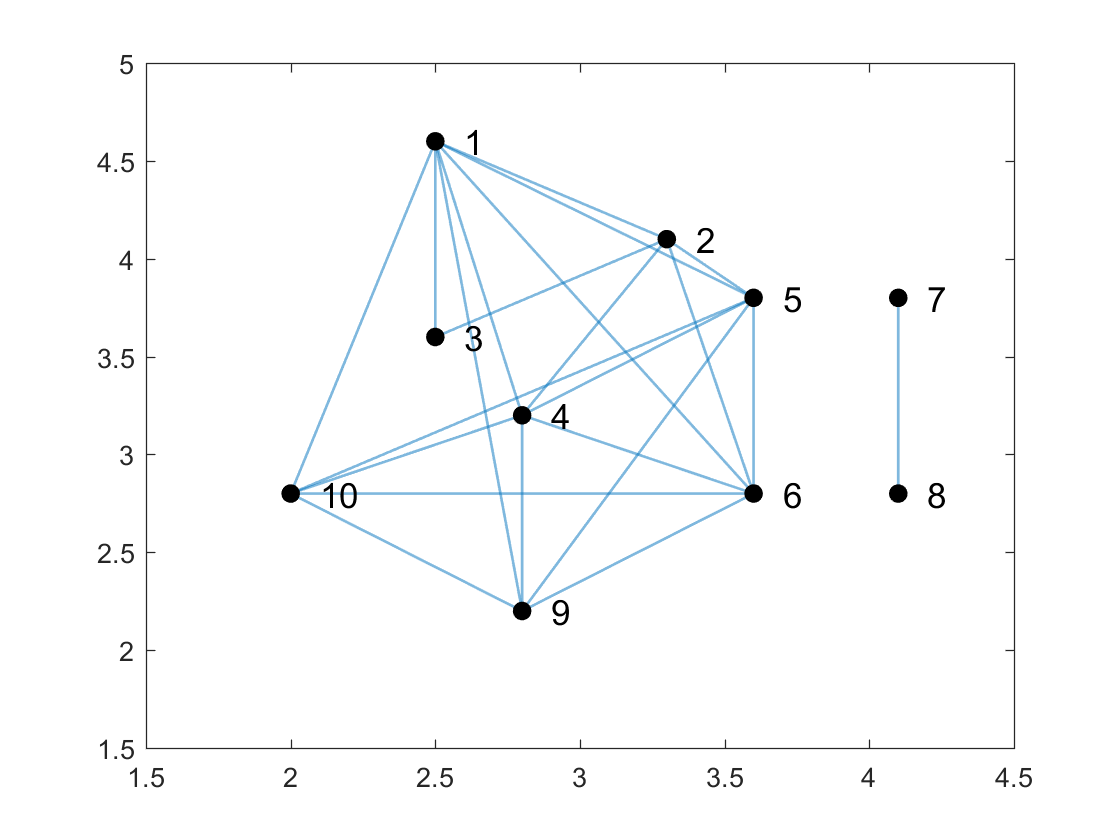}}}%
\hfill
\subfloat[RBF-GAMTL]{{\includegraphics[width=0.23\textwidth]{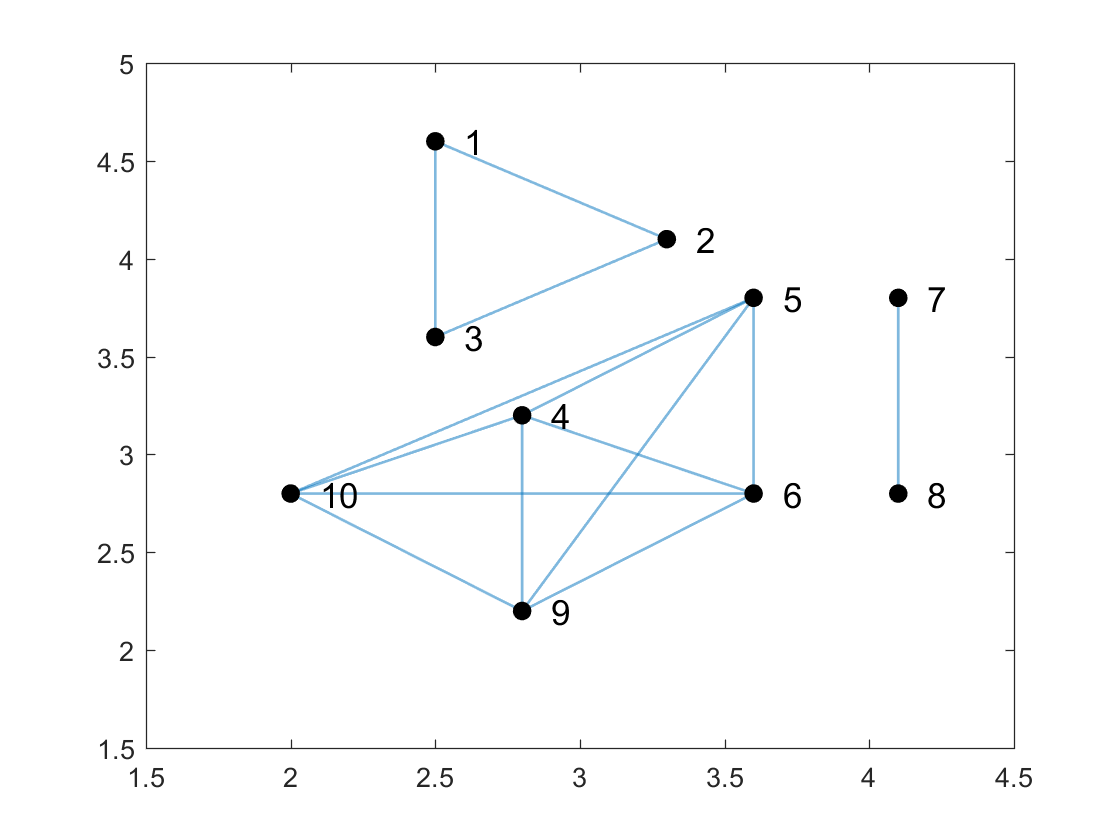}}}
\hfill
\subfloat[CCMTL]{{\includegraphics[width=0.23\textwidth]{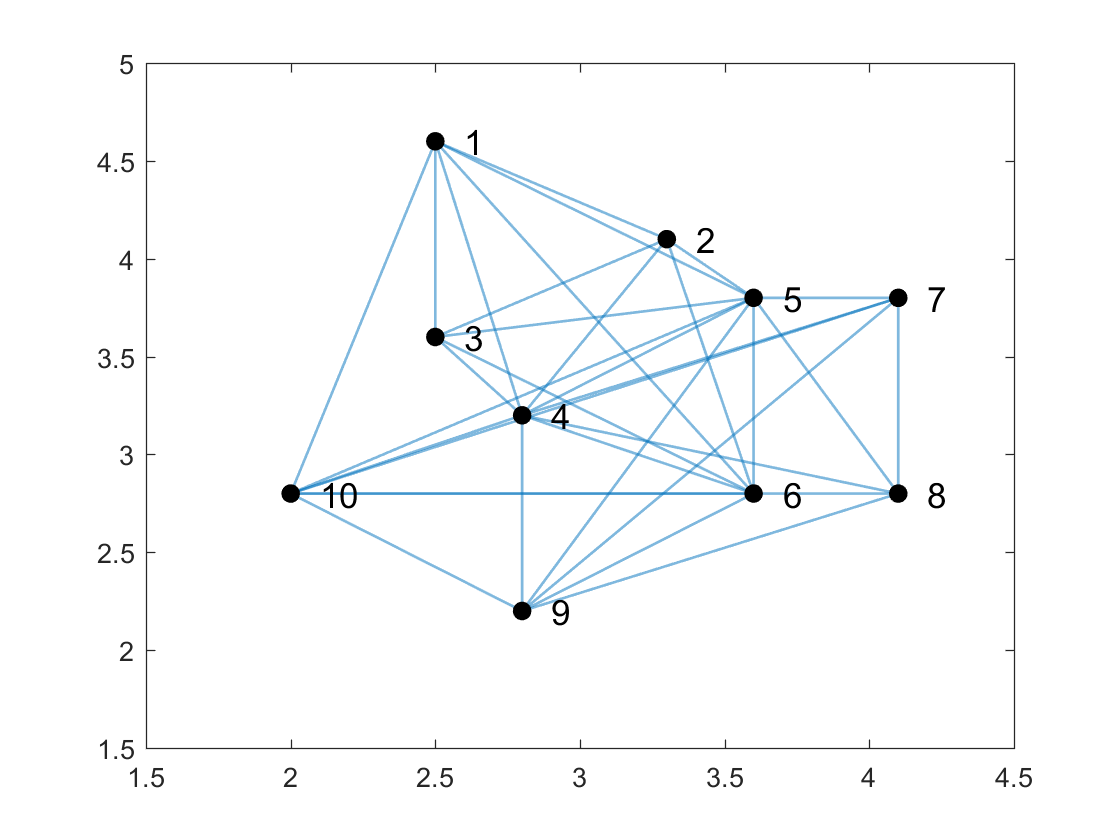} }} \\
\subfloat[TAT]{{\includegraphics[width=0.18\textwidth]{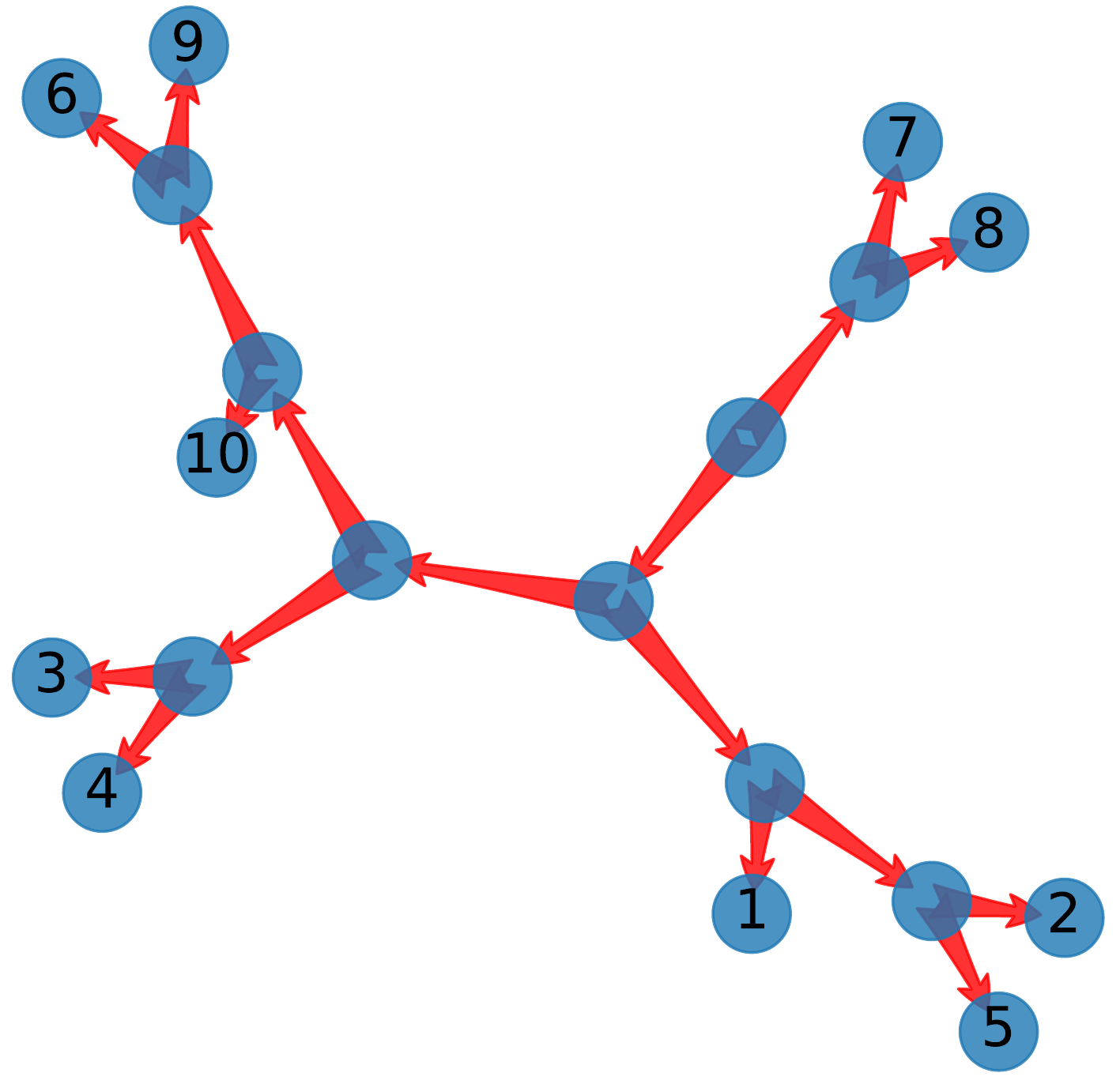}}}
\hfill
\subfloat[BMSL]{{\includegraphics[width=0.23\textwidth]{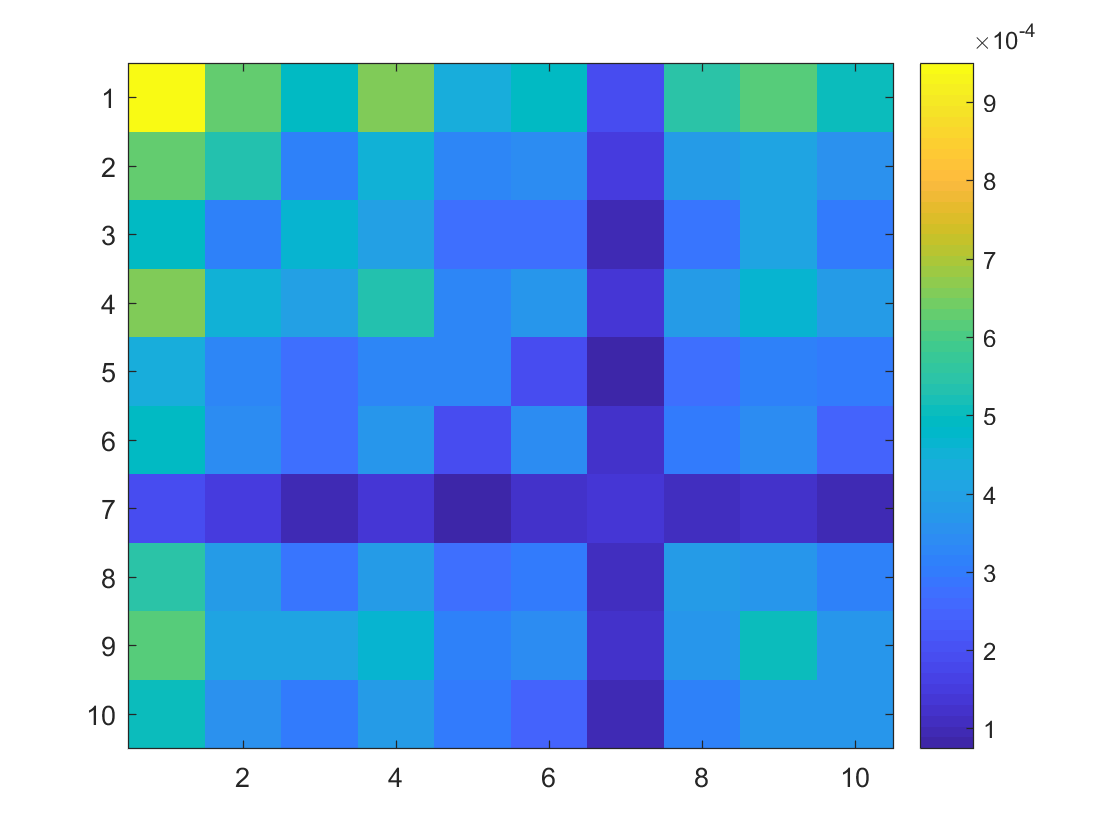} }}%
\hfill
\subfloat[MMSL]{{\includegraphics[width=0.23\textwidth]{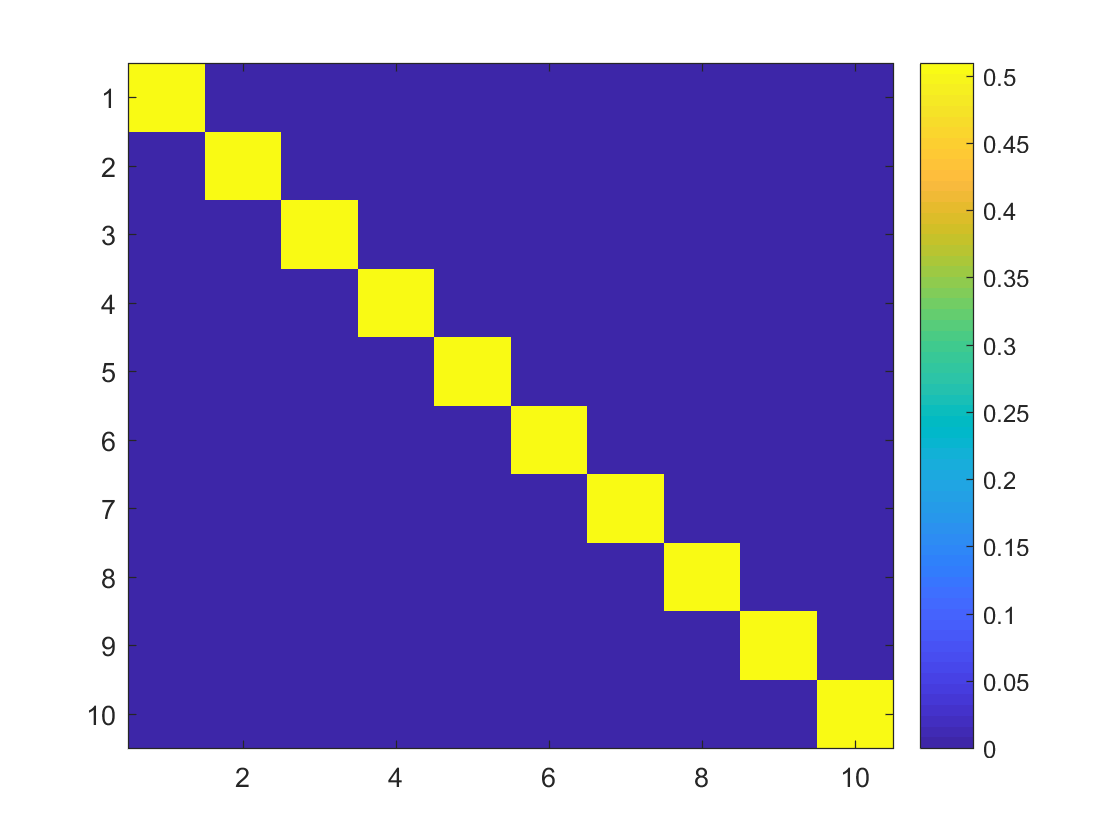} }}%
\hfill
\subfloat[MTRL]{{\includegraphics[width=0.23\textwidth]{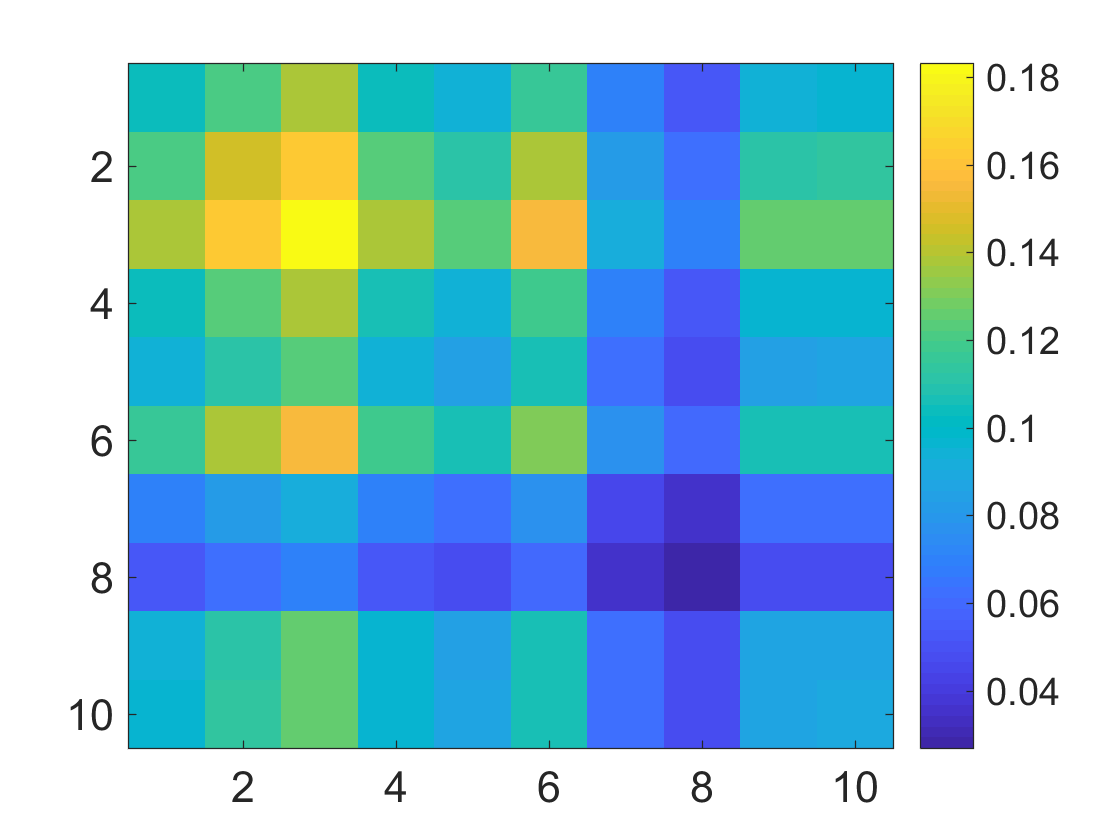} }}%
\caption{(a) networks topology, with $10$ nodes divided into $4$ different clusters (solid lines indicate strong intra-cluster connections, whereas dashed lines indicate weak inter-cluster connections); (b) to (h) demonstrate the task structures on dynamical system data set learned by GAMTL; RBF-GAMTL; CCMTL; TAT; BMSL; MMSL and MTRL, respectively. Graph coordinates are generated by MDS over a dissimilarity matrix evaluated with (symmetric) conditional KL divergence.}
\label{fig:system_graph}
\end{figure*}

\begin{figure}
\setlength{\abovecaptionskip}{0pt}
\setlength{\belowcaptionskip}{-10pt}
\centering
\subfloat[]{{\includegraphics[width=0.45\textwidth]{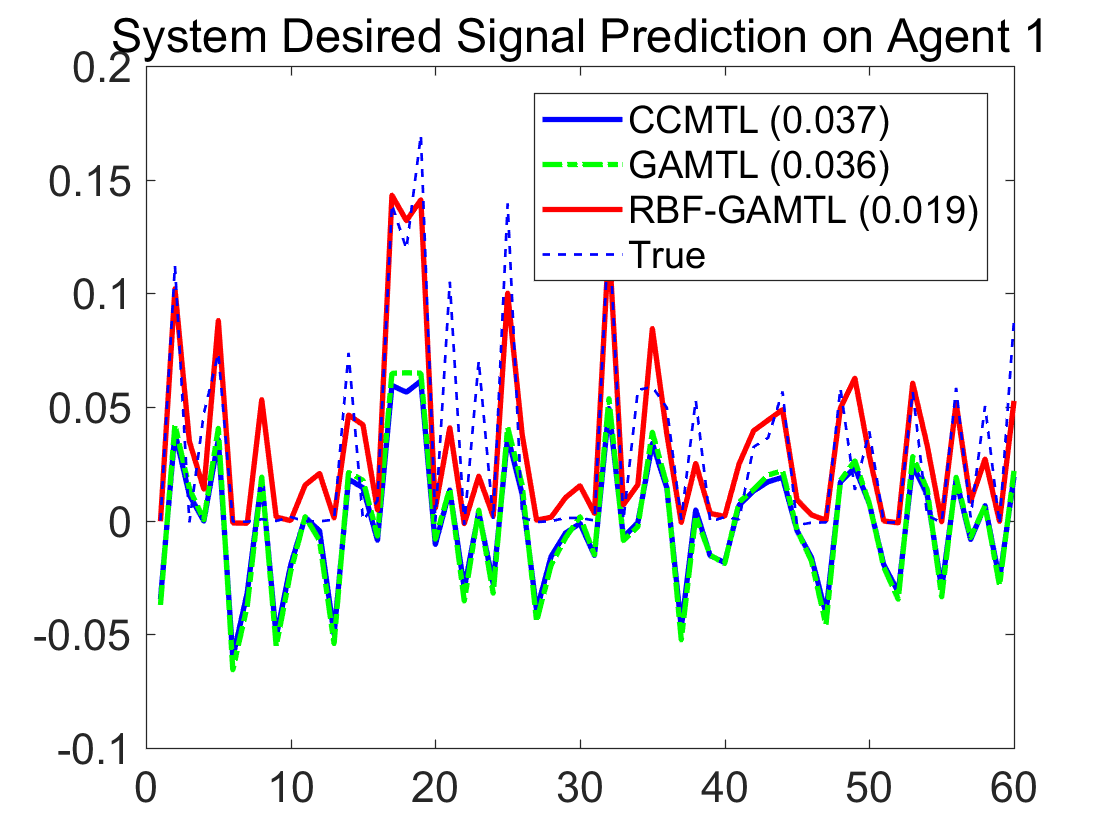} }} \\
\subfloat[]{{\includegraphics[width=0.45\textwidth]{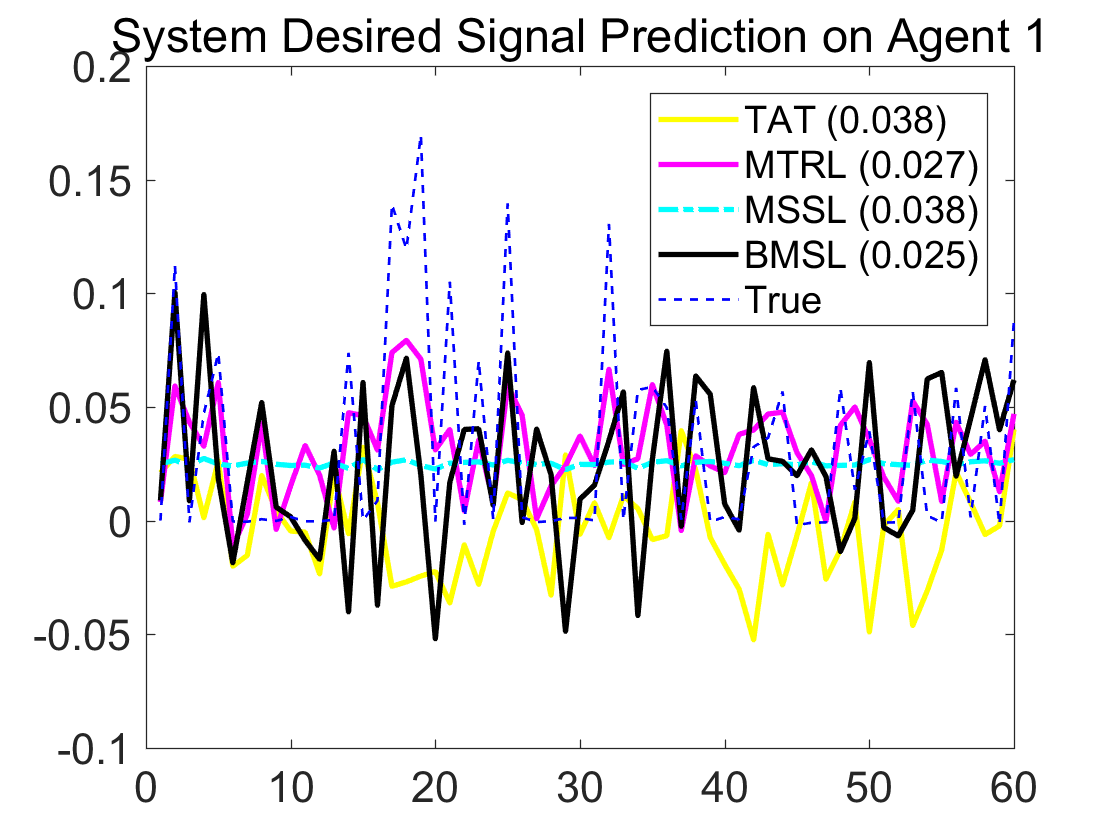} }}
\caption{Prediction results in agent $1$ of each methodology (the value in the bracket indicates RMSE over $10$ agents). Only RBF-GAMTL can track the highly nonlinear dynamics.}
\label{fig:System_info}
\end{figure}

\section{Conclusions and Future Work}


We presented a novel framework for multi-task learning that is able to unveil an easily understandable graph over tasks. The nature of interpretability in this work differs from most existing interpretable machine learning approaches that provide feature level interpretability on revealing how much each feature contributes to the regression/classification result. Our framework provides ``relational interpretability" that exposes how each of the individual task contributes to the performance of a specific task. Besides bringing benefits on interpretability, extensive experiments suggest that our framework is able to reduce the generalization error as well. Finally, to underscore the improved interpretability, we establish the connections between our learned graph and the structure recovered from different machine learning perspectives including information-theoretic learning or time series analysis.


In the future, we will extend the current framework to incorporate feature-level interpretability. We will also consider the joint learning of multiple tasks and other interpretable data structures with more complex intra-group relations.



\bibliographystyle{IEEEtran}
\bibliography{GAMTL}

\end{document}